\documentclass[journal]{IEEEtran}

\IEEEoverridecommandlockouts                              

\usepackage{enumitem}
\usepackage{soul}
\usepackage{cite}
\usepackage{graphicx}
\usepackage{graphics}
\usepackage{amssymb}
\usepackage{amsmath,amsthm}
\usepackage{xspace}
\allowdisplaybreaks  
\usepackage{color}
\usepackage{algpseudocode}
\usepackage{bbm}
\usepackage{comment}
\usepackage{algorithm}
\usepackage{algorithmicx}
\usepackage{subfig}
\usepackage{psfrag}
\usepackage{array}
\usepackage{mathtools}
\usepackage{tikz}
\usetikzlibrary{shapes,arrows}
\usetikzlibrary{decorations.pathreplacing,decorations.markings,decorations.pathmorphing,decorations.shapes}
\usetikzlibrary{positioning,fit}
\usetikzlibrary{calc}
\usepackage{multirow}

\tikzstyle{block}=[draw opacity=0.7,line width=1.4cm]

\definecolor{CranJ}{cmyk}{0,0.69,0.54,0.04} 
\definecolor{PinkJ}{cmyk}{0,0.71,0.43,0.12} 
\definecolor{Cran}{cmyk}{0,0.73,0.41,0.29} 
\definecolor{VRed}{cmyk}{0,0.75,0.25,0.2} 
\definecolor{ORed}{cmyk}{0,0.75,0.75,0} 
\definecolor{CBlue}{cmyk}{1,0.25,0,0} 



\title{
\blue{Server assisted distributed cooperative localization over unreliable communication links}}
\author{Solmaz S. Kia   \quad Jonathan Hechtbauer \quad David Gogokhiya\quad Sonia  Mart{\'\i}nez
  \thanks{The first (the corresponding author) and the third authors
    are, respectively, with the Department of Mechanical and Aerospace Engineering and Department of Computer Science,  
    University of California Irvine, Irvine, CA 92697, USA, {\tt\small
      solmaz,dgogokhi@uci.edu}, the second author is with the Department of Mechatronics, Management Center Innsbruck, Innsbruck, 6020, Austria, and the forth author is
    with the Department of Mechanical and Aerospace Engineering,
    University of California, San Diego, La Jolla, CA 92093, USA,
    {\tt\small soniamd@ucsd.edu} .}
}

\newcommand{\lmssg}{\textsl{Landmark-message}\xspace}

\newcommand{\splitEKF}{\textsl{Split-EKF}\xspace}
\newcommand{\PDsplitEKF}{\textsl{SA-split-EKF}\xspace}

\newcommand{\VV}{\mathcal{V}}

\newcommand{\real}{{\mathbb{R}}}
\newcommand{\reals}{{\mathbb{R}}}

\newcommand{\nonnegativeinteger}{{\mathbb{Z}}_{\ge 0}}

\newcommand{\prpg}{\mbox{-}}
\newcommand{\updt}{\mbox{+}}

\newcommand{\tr}[1]{\text{Trace}(#1)}
\newcommand{\until}[1]{\in\{1,\cdots,#1\}}

\newcommand{\vect}[1]{\boldsymbol{\mathbf{#1}}}
\newcommand{\Bvect}[1]{\bar{\boldsymbol{\mathbf{#1}}}}
\newcommand{\Tvect}[1]{\tilde{\boldsymbol{\mathbf{#1}}}}
\newcommand{\Hvect}[1]{\hat{\boldsymbol{\mathbf{#1}}}}

\newcommand{\Diag}[1]{\operatorname{Diag}(#1)}

\newcommand{\oprocendsymbol}{\hbox{$\bullet$}}
\newcommand{\oprocend}{\relax\ifmmode\else\unskip\hfill\fi\oprocendsymbol}

  \newcommand{\blue}[1]{{\color{black} #1}}

\newtheorem{thm}{Theorem}[section]

\usepackage{tcolorbox}
\definecolor{mycolor}{rgb}{0.122, 0.435, 0.698}
\makeatletter
\newcommand{\mybox}[1]{%
  \setbox0=\hbox{#1}%
  \setlength{\@tempdima}{\dimexpr\wd0+13pt}%
  \begin{tcolorbox}[colframe=mycolor,boxrule=0.5pt,arc=4pt,
      left=6pt,right=6pt,top=6pt,bottom=6pt,boxsep=0pt,width=\@tempdima]
    #1
  \end{tcolorbox}
}
\makeatother

\usepackage[labelsep=endash]{caption}
\begin{document}
\maketitle

\begin{abstract}
  \blue{This paper considers the problem of cooperative localization (CL)
  using inter-robot measurements for a group of networked robots with
  limited on-board resources. We propose a novel recursive algorithm
  in which each robot localizes itself in a global coordinate frame by
  local dead reckoning, and opportunistically corrects its pose
  estimate whenever it receives a relative measurement update message
  from a server.  The
  computation and storage cost per robot in terms of the size of the
  team is of order {$O(1)$}, and the robots are only required to
  transmit information when they are involved in a relative
  measurement. The server also only needs to compute and transmit update
  messages when it receives an inter-robot measurement.  We show that
  under perfect communication, our algorithm is an alternative but
  exact implementation of a joint CL for the entire team via
  Extended Kalman Filter (EKF). The perfect communication however is
  not a hard requirement. In fact, we show that our
  algorithm is intrinsically robust with respect to communication
  failures, with formal guarantees that the updated estimates of the
  robots receiving the update message are of minimum variance in
    a first-order approximate sense at that given timestep.  We
  demonstrate the performance of the algorithm in simulation and~experiments.}
  \end{abstract}
  \vspace{-0.06in} \emph{Keywords}: Cooperative localization; limited
  onboard resources; message dropouts. \vspace{-0.02in}
\section{Introduction}
\vspace{-0.08in} 
\blue{We consider the design of a decentralized cooperative
localization (CL) algorithm for a group of communicating mobile
robots. Using CL, mobile robots in a team improve their positioning
accuracy by jointly processing inter-robot relative measurement
feedbacks.
Unlike classical beacon-based localization algorithms~\cite{JL-HFD:91}
or fixed feature-based Simultaneous Localization and Mapping
algorithms~\cite{MWMGD-PN-SC-HFDW-MC:01}, CL does not rely on external
features of the environment.} As such, this approach is an appropriate
localization strategy in applications that take place in a priori
uncharted environments with no or intermittent GPS
access. 

\blue{Via CL strong correlations among the local state estimates of the
robotic team members are created.  Similar to any state estimation
process, accounting for these cross-correlations is crucial for the
consistency of CL algorithms.  Since correlations create nonlinear
couplings in the state estimate equations of the robots, to
produce consistent results, initial implementations of CL were fully
centralized. These schemes gathered and processed information \emph{at
  each time-step} from the entire team at a single device, either by
means of a leader robot or a fusion center (FC), and broadcast back
the estimated location results to each
robot~\cite{SIR:00,AH-MJM-GSS:02}. Multi-centralized CL, wherein each
robot keeps a copy of the state estimate equation of the entire team
and broadcasts its own information to the entire team so that every
robot can reproduce the centralized pose estimate is also proposed in
the literature~\cite{NT-SIR-GBG:09}. Besides a high-processing cost
for each robot, this scheme requires an all-to-all robot communication
at the time of each information exchange. Developing consistent CL
algorithms that account for the intrinsic cross-correlations of state
estimates with reasonable communication, computation and storage costs
has been an active research area for the past decade.} This problem
becomes more challenging if in-network communications fail due to
external events such as obstacle blocking or limited communication
ranges.

\blue{For applications that maintaining multi-agent connectivity is
challenging,~\cite{AB-MRW-JJL:09,POA-CR-RKM:01,LCC-EDN-JLG-SIR:13,HL-FN:13,DM-NO-VC:13,LL-TS-SIR-WB:16}
propose a set of algorithms in which communication is only required at
the relative measurement times between the two robots involved in the
measurement. As such, these schemes can update only the state estimate
of one or both of these robots. }To eliminate the tight connectivity
requirement, instead of maintaining the exact prior robot-to-robot
correlations, in~\cite{AB-MRW-JJL:09} each robot maintains a bank of
EKFs together with an accurate book-keeping of what robot estimates
were used in the past to update these local filters. Computational
complexity, large memory demand, and the growing size of information
needed at each update time are the main
drawbacks. \blue{In~\cite{POA-CR-RKM:01,LCC-EDN-JLG-SIR:13,HL-FN:13,DM-NO-VC:13},
also the prior robot-to-robot correlations are not maintained, but are
accounted for in an implicit manner using Covariance Intersection
fusion (CIF) method.  Because CIF uses conservative bounds to account
for missing cross-covariance information, these methods often deliver
highly conservative estimates. To improve estimation
accuracy,~\cite{LL-TS-SIR-WB:16} proposes an algorithm in which each
robot, by tolerating an $O(N)$ processing and storage cost, maintains
an approximate track of its prior cross-covariances with
others. 
In another approach to relax
connectivity,~\cite{SEW-JMW-LLW-RME:13} proposes a leader-assistive CL
scheme for underwater vehicles.} This algorithm is a decentralized
extended information filter that uses ranges and state information
from a single reference source (the server) with higher navigation
accuracy to improve localization  
accuracy  of underwater
vehicle(s) (the client(s)). In this scheme the server interacts
with each client separately and there is no cooperation between the
clients.

\blue{Despite their relaxed connectivity requirement, the algorithms
of~\cite{AB-MRW-JJL:09,POA-CR-RKM:01,LCC-EDN-JLG-SIR:13,HL-FN:13,DM-NO-VC:13,LL-TS-SIR-WB:16,SEW-JMW-LLW-RME:13}
are conservative also by nature because they do not enable other
agents in the network to fully benefit from measurement
updates. Recall that correlation terms are means of expanding the
benefit of robot-to-robot measurement updates to the entire team
(see~\cite{SSK-SF-SM:16} for further details). Therefore,
\emph{tightly-coupled} decentralized CL algorithms that maintain
the correlations among the team members result in better localization
accuracy.  One such algorithm obtained from distributing computations
of a joint EKF CL algorithm is proposed in~\cite{SIR-GAB:02}, where
the propagation stage is fully decentralized by splitting each
cross-covariance term between the corresponding two robots.  However,
at update times, the separated parts should be combined, requiring
either an all-to-all robot communications or bidirectional all-to-a
fusion-center communications. Another decentralized CL algorithm based on
decoupling the propagation stage of a joint EKF CL using an
alternative but equivalent formulation of EKF CL is proposed
in~\cite{SSK-SF-SM:16}.  Unlike~\cite{SIR-GAB:02},
in~\cite{SSK-SF-SM:16} each robot can locally reproduce the updated
pose estimate and covariance of the joint EKF at the update stage,
after receiving an update message only from the robot that has made
the relative measurement.  In both of these algorithms, for a team of $N$
robots, each robot incurs an $O(N^2)$ processing and storage cost as
they need to evolve a variable of size of the entire covariance matrix
of the robotic team.}
Subsequently,~\cite{EDN-SIR-AM:09} presents a maximum-a-posteriori
(MAP) decentralized CL algorithm in which all the robots in the team calculate
parts of the centralized CL. All the algorithms above assume that
communication messages are delivered perfectly at all times. A decentralized CL
approach equivalent to a centralized CL, when possible, which handles
both limited communication ranges and time-varying communication
graphs is proposed in~\cite{KYKL-TDB-HHTL:10}. This technique uses an
information transfer scheme wherein each robot broadcasts all its
locally available information (the past and present measurements, as
well as past measurements previously~received from other robots) to
every robot within its communication radius at each time-step. The
main drawback of this algorithm is its high communication and storage
cost.

\blue{ In this paper, we design a novel tightly-coupled distributed CL
  algorithm in which each robot localizes itself in a global
  coordinate frame by local dead reckoning, and
  opportunistically corrects its pose estimate whenever it
  receives a relative measurement update message from a server. The
  update message is broadcast any time server receives an inter-robot
  relative measurement and local estimates from a pair of robots in
  the team that were engaged in a relative measurement. In our setup,
  the server can be a team member with greater processing and storage
  capabilities. Under a perfect communication scenario, we show that
  our algorithm is an exact distributed implementation of a joint
  CL via EKF formulation. To obtain our algorithm, we use an
  alternative representation of EKF formulation of CL called \splitEKF
  for CL.  \splitEKF for CL was proposed in~\cite{SSK-SR-SM:15-icra}
  without the formal guarantee of equivalency. In this paper, we
  establish this guarantee via a mathematical induction proof. Our
  next contribution is to show that our proposed algorithm is robust
  to occasional message dropouts in the network. Specifically, we show
  that the updated estimates of robots receiving the update message
  are minimum variance.  In our algorithm, since every robot only
  propagates and updates its own pose estimates, the storage and
  processing cost per robot is {$O(1)$}. Robots only need to
  communicate with the server if they are involved in an inter-robot
  measurement. Since occasional message drop-outs are allowed in our
  algorithm, the connectivity requirement is flexible. Moreover, we
  make no assumptions about the type of robots or relative
  measurements. Therefore, our algorithm can be employed for teams of
  heterogenous robots.  }

 \vspace{-0.05in}
\section{Preliminaries}\label{sec::robot-discrib}\vspace{-0.08in}
\blue{In this section, we describe our robotic team model and review the joint CL via EKF as well as its alternative representation \splitEKF.  In the proceeding sections, we use \splitEKF to devised our proposed server assisted CL algorithm. 

We consider a team of $N$ robots in which every robot has a detectable
unique identifier and corresponding unique integer label belonging to
the set {$\VV=\{1,\dots,N\}$}. Using a set of
  proprioceptive sensors, robot {$i\in\VV$} measures its self-motion
  and uses it to dead reckon, i.e., propagate its equations of
  motion $\vect{x}^i(k+1)=\vect{f}^i(\vect{x}^i(k),\vect{u}_m^i(k))$,
  $k\in\nonnegativeinteger$, where {$\vect{x}^i\in\reals^{n^{i}}$} is
  the pose vector and
  {$\vect{u}_m^i=\vect{u}^i+\vect{\eta}^i\in\reals^{m^i}$} is the
  measured self-motion variable (for example velocities) with
  $\vect{u}^i$ being the actual value and $\vect{\eta}^i$ the
  contaminating noise. } The robotic team can be heterogeneous.
Every robot also carries exteroceptive
  sensors  to detect, uniquely, the other
  robots in the team and take relative measurements from them, e.g.,
  range or bearing or both. We let ($i\xrightarrow{k}j$) indicate that robot $i$ has taken relative measurement from robot $j$ at time $k$. The relative measurement is modeled by
  {\begin{align}\label{eq::measur_i,j}
      \vect{z}_{i,j}(k)&=\vect{h}_{i,j}(\vect{x}^i(k),\vect{x}^j(k))+\vect{\nu}^i(k),~~\vect{z}_{i,j}\in\real^{n_z^i},
\end{align}}
where {$\vect{h}_{i,j}(\vect{x}^i,\vect{x}^j)$} is the measurement model
and {$\vect{\nu}^i$} is measurement noise. The noises {$\vect{\eta}^i$} and {$\vect{\nu}^i$}, {$i\in\VV$}, are independent
  zero-mean white Gaussian processes with known positive definite
  variances {$\vect{Q}^i(k)=\text{E}[{\vect{\eta}^i}(k)
    \vect{\eta}^i(k)^\top]$} and $\vect{R}^i(k)\!=\!\text{E}[{\vect{\nu}^i}(k)
\vect{\nu}^i(k)^\top]$. 
All noises are assumed to be mutually
uncorrelated. In the~following, we use $\mathbb{S}^{n}_{>0}$ as set of real positive definite $n\times n$ matrices.

\blue{Joint CL via EKF is obtained from applying
EKF over the joint system motion model ~$
  \vect{x}(k\!+\!1)\! =\!(\vect{f}^1(\vect{x}^1, \vect{u}^1),\cdots, \vect{f}^N( \vect{x}^N ,$ $\vect{u}^N))+\Diag{\vect{g}^1(\vect{x}^1),\cdots,\vect{g}^N(\vect{x}^N)}\vect{\eta}(k),$ and  the relative measurement
model~\eqref{eq::measur_i,j}~\cite{SIR-GAB:02}. }
Starting at
$\Hvect{x}^{i\updt}\!(0)\!\in\!\real^{n^i}$, $\vect{P}^{i\updt}\!(0)\!\in\!\mathbb{S}^{n^i}_{>0}$, $\vect{P}_{i,j}^{\updt}(0)\!=\!\vect{0}_{n^i\times
  n^j}$, $i\in\VV$ and $j\!\in\!\VV\backslash\{i\}$, the propagation and update equations of the EKF  CL are 
\begin{subequations}\label{eq::central-robotwise}
 \begin{align}
&    \!\!\!\Hvect{x}^{i\prpg}(k\!+\!1)\!=\vect{f}^i(\Hvect{x}^{i\updt}(k),\vect{u}^i(k)),\label{eq::propag_central_Expanded-a}\\
 &  \!\!\!\vect{P}^{i\prpg}(k\!+\!1)\!=\vect{F}^i(k)\vect{P}^{i\updt}(k)\vect{F}^i(k)\!^\top\!\!\!+\!\vect{G}^i(k)\vect{Q}^i(k)\vect{G}^i(k)\!^\top\!\!\!,\label{eq::propag_central_Expanded-b}\\
  & \!\!\!\vect{P}_{i,j}^{\prpg}(k\!+\!1)\!=\vect{F}^i(k)\vect{P}_{i,j}^{\updt}(k){\vect{F}^j(k)}\!^\top\!\!,\label{eq::propag_central_Expanded-c}\\
     & \!\!\!\Hvect{x}^{i\updt}(k\!+\!1)\!=\Hvect{x}^{i\prpg}(k\!+\!1)+\vect{K}_i(k\!+\!1)\vect{r}^{a}(k\!+\!1),\label{eq::RobotCovarUpdate-a}\\
  & \!\!\!\vect{P}^{\!i\updt}(k\!+\!1)\!=\vect{P}^{\!i\prpg}\!(k\!+\!1)\!-\!\vect{K}_i(k\!+\!1) \vect{S}_{a,b}(k\!+\!1)\vect{K}_i(k\!+\!1)\!\!^\top\!\!\!,\label{eq::RobotCovarUpdate-b}\\
 &  \!\!\!\vect{P}_{\!i,j}^{\updt}(k\!+\!1)\!=\vect{P}_{\!i,j}^{\prpg}(k\!+\!1)\!-\!\vect{K}_{\!i}(k\!+\!1)\vect{S}_{a,b}(k\!+\!1)\vect{K}_{\!j}(k\!+\!1)\!^\top\!\!\!,\label{eq::RobotCovarUpdate-c}\\
   & \!\!\!\vect{K}_i(k\!+\!1)=\label{eq::K_gain_robotwise}\\
   &\begin{cases}\vect{0}, \qquad\qquad\qquad\qquad~\text{no relative measurement at~}k\!+\!1,\\
     (\vect{P}_{i
       ,b}^{\prpg}(k\!+\!1)\Tvect{H}_b^\top\!+\!\vect{P}_{i,a}^{\prpg}(k\!+\!1)\Tvect{H}_a^\top){\vect{S}_{a,b}}^{-1}\!,
    \qquad~ a\xrightarrow{k+1} b.
  \end{cases}\nonumber
\end{align}
\end{subequations} 
 for
$k\in\nonnegativeinteger$, with $\vect{F}^i=\partial\vect{f}(\Hvect{x}^{i\updt},\vect{u}_m^i)/\partial
  \vect{x}^i|_{\Hvect{x}^i,\vect{u}^i_m=\vect{0}}$ and
$\vect{G}^i=\partial\vect{f}(\Hvect{x}^{i\updt},\vect{u}_m^i)/\partial
  \vect{\eta}^i|_{\Hvect{x}^i,\vect{u}^i_m=\vect{0}}$.
Moreover, when a robot $a$ takes a relative measurement from robot $b$ at some given
 time {$k+1$}, the measurement  residual and its covariance are,
 respectively, 
   \vspace{-0.08in}
{\begin{subequations}
\begin{align}
 \!\!\!\vect{r}^{a}(k\!+\!1)\!&=\vect{z}_{a,b}(k\!+\!1)\!-\!\vect{h}_{a,b}(\Hvect{x}^{a\prpg}(k\!+\!1),\Hvect{x}^{b\prpg}(k\!+\!1)),\label{eq::reletive_Residual}\\
  \!\!\vect{S}_{a,b}(k\!+\!1)\!
  \!&=\!\vect{R}^a(k\!+\!1)+\Tvect{H}_a(k\!+\!1)\vect{P}^{a\prpg}(k\!+\!1)\Tvect{H}_a(k\!+\!1)\!^\top\nonumber\\
 &~~+\Tvect{H}_b(k+1) \vect{P}^{b\prpg}(k+1)\Tvect{H}_b(k+1)^\top\label{eq::S_ab}\\ 
&~~+\Tvect{H}_b(k+1)\vect{P}_{ba}^{\prpg}(k+1){\Tvect{H}_a}(k+1)^\top\nonumber\\
&~~+\Tvect{H}_a(k+1)\vect{P}_{a,b}^{\prpg}(k+1)\Tvect{H}_b(k+1)^\top,\nonumber
\end{align}
\end{subequations}} 
where (without loss of generality we assume that
$a<b$) \vspace{-0.08in}\begin{align}
  &\vect{H}_{a,b}(k)=\big[\overset{1}{\vect{0}}~~\overset{\cdots}{\cdots}~~\overset{a}{\Tvect{H}_a}(k)~~\overset{a+1}{\vect{0}}~~\overset{\cdots}{\cdots}~~\overset{b}{\Tvect{H}_b}(k)~~\overset{b+1}{\vect{0}}~~\overset{\cdots}{\cdots}~~\overset{N}{\vect{0}}\big],\nonumber\\
     &\Tvect{H}_l(k)=\partial\vect{h}_{a,b}(\Hvect{x}^{a\prpg}(k),\Hvect{x}^{b\prpg}(k))/\partial \vect{x}^l,\quad l\in\{a,b\}.
  \label{eq::H_ab} 
\end{align} 
\blue{$\vect{P}_{i,j}$~is the cross-covaraince between the estimates~of~robots $i$ and $j$. Equations in~\eqref{eq::central-robotwise} are the representation of the joint EKF CL in robot-wise components, e.g., ~{$\vect{K}\!=\![\vect{K}_1^\top,\cdots,\vect{K}_N^\top
]^\top\!\!=\vect{P}^{\prpg}(k\!+\!1)\vect{H}_{a,b}(k\!+\!1)^\top {\vect{S}_{a,b}}(k+1)^{-1}$} and 
\begin{align}\label{eq::Kgain_central}
  \vect{P}^{\updt}(k\!+\!1)&=\vect{P}^{\prpg}(k\!+\!1)\!-\!\vect{K}(k\!+\!1)\vect{S}_{a,b}\vect{K}(k\!+\!1)^\top
\end{align}  expands as ~\eqref{eq::RobotCovarUpdate-b} and~\eqref{eq::RobotCovarUpdate-c}.  }\\
Since $\vect{K}_{\!i}(k\!+\!1)
\vect{S}_{a,b}(k\!+\!1)\vect{K}_{\!i}(k\!+\!1)^\top$
in~\eqref{eq::RobotCovarUpdate-b} is positive~semi-definite, relative
measurement updates reduce the estimation uncertainty. \blue{However, due to
the inherent coupling in~cross-covariances
~\eqref{eq::propag_central_Expanded-c}
and~\eqref{eq::RobotCovarUpdate-c}, the EKF
CL~\eqref{eq::central-robotwise} can only be implemented in a
  decentralized way using all-to-all communication if each agent
    keeps a copy of its cross-covariance matrices with the rest of the
    team.  \splitEKF~CL, proposed
in~\cite{SSK-SF-SM:16}, is as an \textit{alternative but, as proven
  here, an exactly equivalent} representation of the EKF CL
formulation~\eqref{eq::central-robotwise}.  It uses
a set of intermediate variables to allow for the decoupling of the
estimation equations of the robots as shown in the next~section.}
\begin{thm}[\splitEKF~CL, an exact alternative representation of EKF for joint CL]\label{thm::main}
  Consider the EKF CL algorithm~\eqref{eq::central-robotwise} with
 its given initial conditions.
For $i\in\VV$, let
  $\vect{\Phi}^i(0)=\vect{I}_{n^i}$ and
  $\vect{\Pi}_{i,j}(0)=\vect{0}_{n^i\times n^j}$,
  $j\!\in\!\VV\backslash\{i\}$. Moreover, assume that
  $\vect{F}^i(k)$, {$i\in\VV$}, is invertible at all
  $k\in\nonnegativeinteger$. Next, for $i\in\VV$  let
  \begin{subequations}\label{eq::intermidate_var}
\begin{align}
  &\vect{\Phi}^i(k+1)=\vect{F}^i(k)\vect{\Phi}^i(k),\label{eq::Phi}\\
   &\vect{\Pi}_{i,j}(k\!+\!1)=\vect{\Pi}_{i,j}(k)+\vect{\Gamma}_i(k\!+\!1)\,\vect{\Gamma}_j(k\!+\!1)\!^\top,\label{eq::Pi}
  \end{align}
  \end{subequations}
  $j\!\in\!\VV\backslash\{i\}$, where
    \begin{subequations}\label{eq::intermidate_var2}
\begin{align}
  &\vect{\Gamma}_{i}(k\!+\!1)=\vect{0},~~~~~~~
  \text{no relative measurement at~} k\!+\!1,\label{eq::barD-no-meas}\\
  &\vect{\Gamma}_{a}(k\!+\!1)=
  \big(\vect{\Pi}_{a,b}(k){\vect{\Phi}^b(k+1)}^\top\Tvect{H}_{b}^\top\!
  +\nonumber\\
  &\qquad\!\vect{\Phi}^a(k+1)^{-1}\vect{P}^{a\prpg}(k+1)\Tvect{H}_{a}^\top\big)\,{\vect{S}_{a,b}}\!
  \!^{-\frac{1}{2}}\!,~\quad a\xrightarrow{k+1} b, \label{eq::barD-a}\\
  &\vect{\Gamma}_{b}(k+1)=
  \big(\vect{\Phi}^{b}(k+1)^{-1}\vect{P}^{b\prpg}(k+1)\Tvect{H}_{b}^\top\!
  +\nonumber\\
  &~~~~~~\qquad\!\vect{\Pi}_{b,a}(k){\vect{\Phi}^a(k+1)}^\top\Tvect{H}_{a}^\top\big)
  \,{\vect{S}_{a,b}}\!\!^{-\frac{1}{2}},\quad a\xrightarrow{k+1} b, \label{eq::barD-b}\\
  &\vect{\Gamma}_{l}(k+1)=
  (\vect{\Pi}_{l,b}(k){\vect{\Phi}^b(k+1)}^\top\Tvect{H}_{b}^\top\!
  +\vect{\Pi}_{l,a}(k) \times\nonumber\\
  &~~~~~~{\vect{\Phi}^a(k+1)}^\top\Tvect{H}_{a}^\top)\,{\vect{S}_{a,b}}\!\!^{-\frac{1}{2}},
  ~ l\!\in\!\VV\backslash\{a,\!b\}, ~~ a\xrightarrow{k+1} b,\label{eq::barD-i}
\end{align} 
\end{subequations}
for $k\in\nonnegativeinteger$. Then, we can write~\eqref{eq::propag_central_Expanded-c}~as
\begin{align}
  \vect{P}_{i,j}^{\prpg}(k+1)=&\vect{\Phi}^i(k+1)\,
  \vect{\Pi}_{i,j}(k)\,\vect{\Phi}^j(k+1)^\top,\label{eq::alternate-EKF-equations-a}
 \end{align}
and~\eqref{eq::RobotCovarUpdate-a}, \eqref{eq::RobotCovarUpdate-b}  and
\eqref{eq::RobotCovarUpdate-c}, respectively,~as
\begin{subequations}\label{eq::alternative-EKF-update-xi-Pi}
\begin{align}
  \!\!\Hvect{x}^{i\updt}\!(k\!+\!1)\!=\,&
  \Hvect{x}^{i\prpg}\!(k\!+\!1)\!+\!\vect{\Phi}^{i}\!(k\!+\!1)\vect{\Gamma}_i(k\!+\!1)
  \Bvect{r}^{a}\!(k\!+\!1),\label{eq::alternate-EKF-equations-b}\\
  \!\!\vect{P}^{i\updt}(k\!+\!1)\!=\,&\vect{P}^{i\prpg}(k+1)-\label{eq::alternate-EKF-equations-c}\\
  &\vect{\Phi}^{i}(k+1)
  \vect{\Gamma}_{i}(k+1)\vect{\Gamma}_i^\top(k+1)
  \vect{\Phi}^i(k+1)^\top\!\!\!,\nonumber\\
    \vect{P}_{i,j}^{\updt}(k+1)=\,&\vect{\Phi}^i(k+1)\,\vect{\Pi}_{i,j}(k+1)\,
  \vect{\Phi}^j(k+1)^\top,\label{eq::alternate-EKF-equations-d}
 \end{align}
\end{subequations}
 for $i\in\VV$ and $j\!\in\!\VV\backslash\{i\}$, where
$\Bvect{r}^{a}(k\!+\!1)={\vect{S}_{a,b}}\!\!^{-\frac{1}{2}}\vect{r}^{a}(k\!+\!1)$. 
\end{thm}
\blue{The proof of this theorem is given in Appendix. Inevitability of $\vect{F}^i(k)$ is generic and holds for a wide
  class of motion models e.g., non-holonomic robots. Note here that using~\eqref{eq::alternate-EKF-equations-d}, $\vect{S}_{a,b}$ in~\eqref{eq::S_ab} can be expressed equivalently as}
\begin{align}\label{eq::Sab_DCL}
  \vect{S}_{a,b}=&\,\vect{R}^{a}(k\!+\!1)\!+\!\Tvect{H}_{a}\vect{P}^{a\prpg}(k\!+\!1)\Tvect{H}_{a}^\top\!+\!\Tvect{H}_{b}
  \vect{P}^{b\prpg}(k\!+\!1)\Tvect{H}_{b}^\top\nonumber\\
  &+\Tvect{H}_{a}\vect{\Phi}^a(k+1)\vect{\Pi}_{a,b}(k)\vect{\Phi}^b(k\!+\!1)^\top\Tvect{H}_{b}^\top+\\
  &\Tvect{H}_{b}\vect{\Phi}^b(k\!+\!1)\vect{\Pi}_{b,a}(k){\vect{\Phi}^a(k\!+\!1)}^\top\Tvect{H}_{a}^\top.\nonumber
  \end{align}

\vspace{-0.05in}
\section{\blue{A server assisted distributed cooperative localization}}\label{sec::partially-decentralized}\vspace{-0.08in}
\blue{In this section, we propose a novel distributed cooperative localization algorithm  
in which each agent maintains its
  own local state estimate for autonomy, incurs only $O(1)$ processing and
  storage costs, and needs to communicate only when there is an
  inter-agent relative measurement. Our proposed solution is a server assisted distributed implementation of \splitEKF~CL (\PDsplitEKF for short)  which is given in Algorithm~\ref{alg::ouralgpar}.  For clarity of presentation, we are assuming that at most there
  is one relative measurement at each time in the team. To process
  multiple synchronized measurements, we use \emph{sequential
    updating} (c.f.~\cite[ch. 3]{CTL:66},\cite{YB-PKW-XT:11}), for details see Appendix. 

  In \PDsplitEKF, every robot $i\in\VV$ maintains
  and propagates its own propagated state estimate~\eqref{eq::propag_central_Expanded-a} and covariance
  matrix~\eqref{eq::propag_central_Expanded-b}, as well as, the
  variable $\vect{\Phi}^i\in\real^{n^i\times n^i}$~\eqref{eq::Phi}. Since these variables are
  local, the propagation stage is fully decoupled and there is no need for communication at this stage. To free the
  robots from maintaining the team cross-covariances, \PDsplitEKF assigns a server to
  maintain and to update $\vect{\Pi}_{i,j}$'s~\eqref{eq::Pi},
  the main source of high processing and storage costs. 
  The communication between robots and the server is only required when
  there is a relative measurement in the~team. 
  When robot $a$ takes relative measurement from robot $b$,
robot $a$ informs the server.  Then, the server starts the update procedure
by taking the following actions. First, it acquires the $\lmssg$~\eqref{eq::DCL-lmssg}
from robots $a$ and $b$, which 
is of order $O(1)$ in terms of the size of the team.  
Then, using this
information along with its locally maintained $\vect{\Pi}_{i,j}$'s, server calculates
and sends to
each robot $i\in\VV$ its corresponding update message~\eqref{eq::updt_mssg}
so that the robot can update its local estimates
using~\eqref{eq::alternative-EKF-update-xi-Pi}. 
It also updates its
local $\vect{\Pi}_{i,j}$ using \eqref{eq::Pi}, for all
$i\in\VV\backslash\{N\}$ and $j\in\{i+1,\cdots,N\}$--because of the
symmetry of the joint matrix $\vect{\Pi}$ we only save the upper triangular part of this
matrix. The size of update message for each robot is of order $O(1)$ in terms of the size of the team. We can show that multiple concurrent measurements can be processed jointly at the server and the update message for each robot is still of order $O(1)$, for details see Appedix.} \PDsplitEKF~CL algorithm processes absolute measurements in a similar way to relative measurements, i.e., the robot with the absolute measurement informs the server, which proceeds with the same described updating procedure
and issues the update message~\eqref{eq::updt_mssg} to every
robot~$i\!\in\!\VV$. 
 
 \blue{A fully decentralized implementation of the
 \splitEKF~CL has been proposed in~\cite{SSK-SF-SM:16}. In this scheme, instead of a server each agent keeps a local copy of $\vect{\Pi}_{l,j}(k)$'s which results in an {$O(N^2)$} storage and
{$O(N^2\times N_z)$} processing cost per robot with {$N_z$} the total
number of relative measurement in the team in a given time.  The
downside of the algorithm of~\cite{SSK-SF-SM:16} is that any incidence
of message dropout at each agent causes disparity between the local
copy of $\vect{\Pi}_{l,j}(k)$'s at that agent and the local copies of
the rest of the team, jeopardizing the integrity of the
decentralized~implementation.  In the next section we show that \PDsplitEKF has robustness to message dropouts.
  }


    
    
\section{Accounting for in-network message
  dropouts}\label{sec::C-CL_miss}\vspace{-0.08in}
\blue{\PDsplitEKF CL described so far operates based on the assumption that
at the time of measurement update, all the robots can receive the
update message of the server, i.e., $\VV_{\text{missed}}(k+1)$, the set
of agents missing the update message of the server at timestep $k+1$, is
empty.  It is straightforward to see that \PDsplitEKF~CL algorithm is robust to permanent
team member dropouts. The server only suffers from a processing and
communication cost until it can confirm that the dropout is
permanent. In what follows, we study the robustness of Algorithm~\ref{alg::ouralgpar} against occasional
communication link failures between robots and the server. Specifically we show that 
Algorithm~\ref{alg::ouralgpar} has robustness to message dropout with formal guarantees that the updated estimates of the
  robots receiving the update message are of minimum variance in
    a first-order approximate sense at that given timestep.}
Our guarantees are based on the assumption that the two
robots involved in a relative measurement can both communicate with
the server at the same time otherwise, we discard that measurement. We
base our study on analyzing a EKF for joint CL in which at some update
times, we do not update the estimate of some of the robots. In our
server assisted distributed  implementation, these robots
are those which miss the update-message of the server and as such they
are not updating their~estimates.
    
\setlength{\textfloatsep}{5pt}
\begin{algorithm}[!t]
{\scriptsize
\caption{{ \PDsplitEKF~CL}}
\label{alg::ouralgpar}
\begin{algorithmic}[1]
\Require
 Initialization ($k=0$): 
\begin{align*}
&\text{Robot~} i\in\VV:~ \Hvect{x}^{i\updt}(0)\in\real^{n^i},~~\vect{P}^{i\updt}(0)\in\mathbb{S}_{>0}^{n^i},~~\vect{\Phi}^i(0)=\vect{I}_{n^i},\\
&~~~~~\quad\text{Server}:~\vect{\Pi}^i_{i,j}(0)=\vect{0}_{n^i\times n^j}, ~i\in\VV\backslash\{N\},~~j\in\{i+1,\cdots,N\}.
\end{align*}

\hspace{-0.38in}\noindent\textbf{Iteration $k$}
\State \textbf{Propagation}: Every robot $i\in\VV$ proceeds by
$$(\Hvect{x}^{i\prpg}\!(k\!+\!1), \vect{\Phi}^i(k\!+\!1),\vect{P}^{i\prpg}\!(k\!+\!1) )\xleftarrow[\eqref{eq::propag_central_Expanded-a}, \eqref{eq::propag_central_Expanded-b},\eqref{eq::Phi}]{\text{using}}(\Hvect{x}^{i\updt}(k),\vect{\Phi}^i(k),\vect{P}^{i\updt}(k),\vect{u}_{m}^i(k) ).$$

\State \textbf{Update}:
\begin{itemize}[leftmargin=*]
\item if there is  no relative measurements in the network
\begin{align*}
&\text{Robot~} i\!\in\!\VV:~ \Hvect{x}^{i\updt}(k+1)=\Hvect{x}^{i\prpg}(k+1),~\vect{P}^{i\updt}(k+1)=\vect{P}^{i\prpg}(k+1),\\
&~~~~\quad\text{Server}:~\vect{\Pi}_{i,j}(k+1)=\vect{\Pi}_{i,j}(k), ~~i\!\in\!\VV,~j\in\VV\backslash\{i\}.
\end{align*}
\item if $a\xrightarrow{k+1}b$, $a$ informs the server. The server asks for the following information from robots $a$ and $b$, respectively, 
\begin{align}\label{eq::DCL-lmssg}
&\lmssg^a=\Big(\vect{z}_{a,b},\Hvect{x}^{a\prpg}(k+1), \vect{P}^{b\prpg}(k+1), \vect{\Phi}^{a}(k+1)\Big),\nonumber\\
&\lmssg^b=\Big(\Hvect{x}^{b\prpg}(k+1),\vect{P}^{b\prpg}(k+1), \vect{\Phi}^{b}(k+1)\Big).
\end{align} 
Then, server compute 
$$\vect{S}_{a,b},\vect{r}^{a},\vect{\Gamma}_{i}\xleftarrow[\eqref{eq::Sab_DCL},\eqref{eq::reletive_Residual},\eqref{eq::intermidate_var2}]{\text{using}}(\lmssg^a,\lmssg^b).$$
and $\Bvect{r}^{a}\!=\!(\vect{S}_{a,b})^{-\frac{1}{2}}\vect{r}^{a}$. Server passes the following data to every robot $i\in\VV$,
\begin{align}\label{eq::updt_mssg}
&\textsl{update-message}^i=\big(\Bvect{r}^{a}\,,\,\vect{\Gamma}_i\big).
\end{align}
Robot $i\in\VV\backslash\VV_{\text{missed}}(k+1)$ then updates its local state estimate according to 
\begin{subequations}
\begin{align}
\Hvect{x}^{\!i\updt}(k+1)=\,&\Hvect{x}^{i\prpg}(k+1)+\\
&\vect{\Phi}^{i}(k\!+\!1)\textsl{update-message}^i(2)\,\textsl{update-message}^i(1),\nonumber\\ 
 \vect{P}^{\!i\updt}(k+1)=\,&\vect{P}^{i\prpg}\!(k+1)-\\
 &\vect{\Phi}^{\!i}(k\!+\!1)\textsl{update-message}^i(2)\,\textsl{update-message}^i(2)^\top\vect{\Phi}^i\!(k\!+\!1)^{\!\top}.\nonumber
\end{align}
\end{subequations}
The server updates its local variables, for $i\in\VV\backslash\{N\},~j\in\{i+1,\cdots,N\}$:
\begin{align*}
\vect{\Pi}_{i,j}(k\!+\!1)&=\vect{\Pi}_{i,j}(k)-\vect{\Gamma}_i \vect{\Gamma}_j^\top, ~ \text{if~}(i,j)\not\in\VV_{\text{missed}}(k\!+\!1)\!\times\!\VV_{\text{missed}}(k\!+\!1).
\end{align*}
\end{itemize}
\State $k \leftarrow k+1$
\end{algorithmic}
$\VV_{\text{missed}}(k+1)$ is the set of agents missing the update message at timestep $k+1$.
}
\end{algorithm}

\blue{In what follows, the state estimate equations of the robots involved
in a relative measurement do always get updated. }
Without loss of generality, assume that we do not update the state estimate of robots
{$\{m+1,\cdots, N\}$}, for $2<m<N+1$ using the relative measurement
taken by robot {$a \until{m}$} from robot {$b \until{m}$} at some time
$k+1$. That is, assume that agents
{$\VV_{\text{missed}}(k+1)=\{m+1,\cdots, N\}$} have missed the update
of the server at time $k+1$.  The propagation stage of the Kalman filter
is independent of the observation process, and thus we leave it as
is, see~\eqref{eq::propag_central_Expanded-a}-\eqref{eq::propag_central_Expanded-c}. The
following result gives the minimum variance update equation for robots
$\{1,\cdots,m\}$.  Recall that, at any update incident at timestep
$k$, the EKF gain {$\vect{K}$} minimizes
$\text{Trace}(\vect{P}^{\updt}(k))$, where $\vect{P}^{\updt}(k)$
in~\eqref{eq::Kgain_central} is an approximation of
$\text{E}[(\vect{x}(k)-\vect{x}^{\updt}(k))(\vect{x}(k)-\vect{x}^{\updt}(k))^\top]$--an
approximation based on a system and measurement model linearization
(c.f.~\cite[page 146]{JLC-JLJ:11}). \blue{The following result plays a
similar role.}

\begin{thm}[\blue{Joint EKF CL in the presence of message dropouts}]\label{thm::partial_update}
 Consider a joint CL via EKF where the
  relative measurement taken by robot
  $a\notin\VV_{\text{missed}}(k+1)$
  from robot $b\notin \VV_{\text{missed}}(k+1)$
  ~at some time $k+1>0$ is used to only update the states of robots
  $\VV\backslash\VV_{\text{missed}}(k\!+\!1)=\{1,\cdots,m\}$, i.e.,
  \begin{subequations}\label{eq::update-cent-miss}
    \begin{align}
      \Hvect{x}^{i\updt}(k\!+\!1)=&\Hvect{x}^{i\prpg}(k\!+\!1)+
      \vect{K}_i(k\!+\!1)\vect{r}^{a}(k\!+\!1),\nonumber\\
      &
     \qquad \qquad\quad i\in\VV\backslash\VV_{\text{missed}}(k+1)\label{eq::XRobotCovarUpdate-1:N-1}\\
      \Hvect{x}^{i\updt}(k\!+\!1)=&\Hvect{x}^{i\prpg}(k\!+\!1)\quad
      i\in\VV_{\text{missed}}(k+1).\label{eq::RobotCovarUpdate-N}
    \end{align} 
  \end{subequations}
  Let
  $\vect{K}_{1:m}=[\vect{K}_1^\top,\cdots,\vect{K}_m^\top]^\top$. Then,
  the Kalman gain $\vect{K}_{1:m}$ that minimizes $\text{Trace}(\vect{P}^{\updt}(k+1))$, for $i\!\in\!\VV\backslash\VV_{\text{missed}}(k+1)$, is 
  \begin{align}\label{eq::gain_partial_update}
    \vect{K}_i=(\vect{P}_{i, b}^{\prpg}(k+1)
    \Tvect{H}_b^\top+&\vect{P}_{i,a}^{\prpg}(k+1)\Tvect{H}_a^\top)\,{\vect{S}_{a,b}}^{-1}.\end{align}
  Moreover, the team covariance update is given by
  \begin{subequations}
 \begin{align}
    &\vect{P}^{i\updt}(k\!+\!1)\!=\label{eq::robot-covar-updt-miss}\\~~&\begin{cases}
    \vect{P}^{i\prpg}(k\!+\!1),~~~~~\qquad
    i\in\VV_{\text{missed}}(k+1),\\
    \vect{P}^{i\prpg}(k\!+\!1)\!-
    \!\vect{K}_i
    \vect{S}_{a,b}(k\!+\!1)\vect{K}_i(k\!+\!1)^\top,~\text{otherwise}.
    \end{cases}\nonumber\\
     & \vect{P}_{i,j}^{\updt}(k\!+\!1)\!  =\label{eq::robot-cross-covar-updt-miss} \\
    ~~&\begin{cases}
      \vect{P}_{i,j}^{\prpg}(k\!+\!1),~~~~~\quad (i,j)\in\VV_{\text{missed}}(k\!+\!1)\times\VV_{\text{missed}}(k\!+\!1),\\
      \vect{P}_{i,j}^{\prpg}(k\!+\!1)\!-\!\vect{K}_i(k\!+\!1),
      \vect{S}_{a,b}(k\!+\!1)\vect{K}_j(k\!+\!1)^\top,~\text{otherwise}.
    \end{cases}\nonumber
  \end{align} 
  \end{subequations}
     where for $i\!\in\!\VV_{\text{missed}}(k\!+\!1)$ we defined and used the \emph{pseudo}~gain
\begin{equation}\label{eq::pseudo-gain}
  \vect{K}_i=(\vect{P}_{i,b}^{\prpg}(k+1)
  \Tvect{H}_b^\top+\vect{P}_{i,a}^{\prpg}(k+1)\Tvect{H}_a^\top)
  \,{\vect{S}_{a,b}}^{-1}.  
\end{equation} 
\end{thm}
\blue{The proof of this theorem is given in Appendix. The partial updating equations~\eqref{eq::update-cent-miss}-\eqref{eq::pseudo-gain} are the same as the joint EKF CL~\eqref{eq::central-robotwise} except that the state estimate and corresponding covariance matrix for agents missing the update message and also the cross-covariance matrices between those agents do not get updated. 
As such, the \splitEKF representation for~\eqref{eq::update-cent-miss}-\eqref{eq::pseudo-gain} is the same as the one for the joint EKF CL~\eqref{eq::central-robotwise} except that for   $i\in\VV_{\text{missed}}(k+1)$ we have 
\begin{align*}
\Hvect{x}^{i\updt}(k\!+\!1)=&\Hvect{x}^{i\prpg}(k\!+\!1), ~\vect{P}^{i\updt}(k+1)=\vect{P}^{i\prpg}(k+1),\\
\vect{\Pi}_{i,j}(k\!+\!1)=&\vect{\Pi}_{i,j}(k),\quad\quad\quad j\!\in\!\VV_{\text{missed}}(k\!+\!1)\backslash\{i\}.
\end{align*}
Therefore, for none empty $\VV_{\text{missed}}(k+1)$, we can implement the \PDsplitEKF~CL algorithm
exactly as described in Algorithm~\ref{alg::ouralgpar}.
We conclude then that  \PDsplitEKF~CL algorithm is robust to
message dropouts and the estimates of the robots receiving the update
message, as stated above, are minimum~variance, in a first-order
approximate sense.  }

\section{Numerical and experimental evaluations}
\blue{We demonstrate the performance of the proposed \PDsplitEKF~CL algorithm
with and without occasional communication failure in simulation and compare it to the performance of dead reckoning only localization and that of the algorithm of~\cite{HL-FN:13}. We use a team of four robots moving on a flat terrain 
on the square
helical paths shown in Fig.~\ref{fig::simulation} (a) and (b)  traversed in $[0,300]$ seconds (crosses show the start points). 
The standard deviation of the linear
(resp. rotational) velocity measurement noise of robots $\{1,2,3,4\}$ respectively are assume
to be {$\{35\%,30\%,25\%,20\%\}$} of the linear (resp. {$\{25\%,20\%,20\%,15\%\}$} of the rotational)
velocity of the robot. 
For the measurement/communication scenario in Table~\ref{table::simulation_time}, the root mean square (RMS) position error calculated from {$M\!=\!50$} Monte Carlo runs is depicted  in Fig.~\ref{fig::simulation}  (c)-(f).
As seen, in comparison to dead reckoning localization, CL improves the accuracy of the state estimates. As expected, by keeping an accurate account of the cross covariances, the \PDsplitEKF CL algorithm produces more accurate localization results than the algorithm of
~\cite{HL-FN:13}. Recall that the advantage of the algorithm of~\cite{HL-FN:13} is its relaxed connectivity condition. However, since this algorithm accounts for missing cross-covariance information by conservative estimates, its localization accuracy suffers. Also in this algorithm since only the landmark robots (the robots that relative measurements are taken from them) update their estimates, the robots taking the relative measurement does not benefit from CL. Fig.~\ref{fig::simulation}  (c)-(f) also demonstrate the robustness of  \PDsplitEKF CL to communication failure, i.e., the robots receiving the update message benefit from CL and the disconnected robot once connected can resume correcting its state estimates. Here, it is also worth recalling that \PDsplitEKF~CL without link failure, similar to algorithms of~\cite{SIR-GAB:02} and \cite{SSK-SF-SM:16},  recovers exactly the state estimate of the joint EKF CL~\eqref{eq::central-robotwise}.  However, unlike the algorithms of ~\cite{SIR-GAB:02} and \cite{SSK-SF-SM:16} \PDsplitEKF~CL has robustness to the communication failure. }

\blue{
\setlength\extrarowheight{4pt}
 \begin{table}[b]\renewcommand{\arraystretch}{0.7}\scriptsize
    \caption{{\small Time table for exteroceptive 
      measurement times and the disconnected robots. 
      }
    }\label{table::simulation_time}\vspace{-0.06in}
  \centering
  \begin{tabular}{| m{1.3cm} ||  m{0.5cm} | m{0.5cm} |m{0.8cm}|m{0.8cm}|m{0.8cm}|m{0.8cm}|}
    \hline
    Time (sec.)&$\!\!\!\!(45,50]$&$\!\!\!\!(90,95]$&$\!\!\!\!(135,140]$&$\!\!\!\!(180,185]$&$\!\!\!\!(225,230]$&$\!\!\!\!(270,275]$\\ \hline
    Measurements &
    $\!\!\!\!\!\!\!\left.\begin{array}{c}1\to 2\\ 2\to 3\\3\to 4 \end{array}\right.$&
    $\!\!\!\!\!\!\!\left.\begin{array}{c}3\to 4\\ 4\to1 \end{array}\right.$&
    $\!\!\!\!\left.\begin{array}{c}1\to 2\\ 3\to 4 \end{array}\right.$&
    $\!\!\!\!\left.\begin{array}{c}2\to 3 \end{array}\right.$&
    $\!\!\!\!\left.\begin{array}{c}1\to 2\\ 3\to 4\end{array}\right.$&
    $\!\!\!\!\left.\begin{array}{c}2\to 3\\ 4\to 1\end{array}\right.$
    \\ \hline
   \!\! disconnected from server& none& none& robot $4$&robot $4$&none&none\\       \hline
    \end{tabular}\vspace{-0.1in}
\end{table}

}
\begin{figure}[t!]
  \unitlength=0.5in
  \centering 
  \subfloat[true trajectory of robots 1 and 2]{
    \!  \includegraphics[trim=3 2 4 5,clip,width=0.25\textwidth]{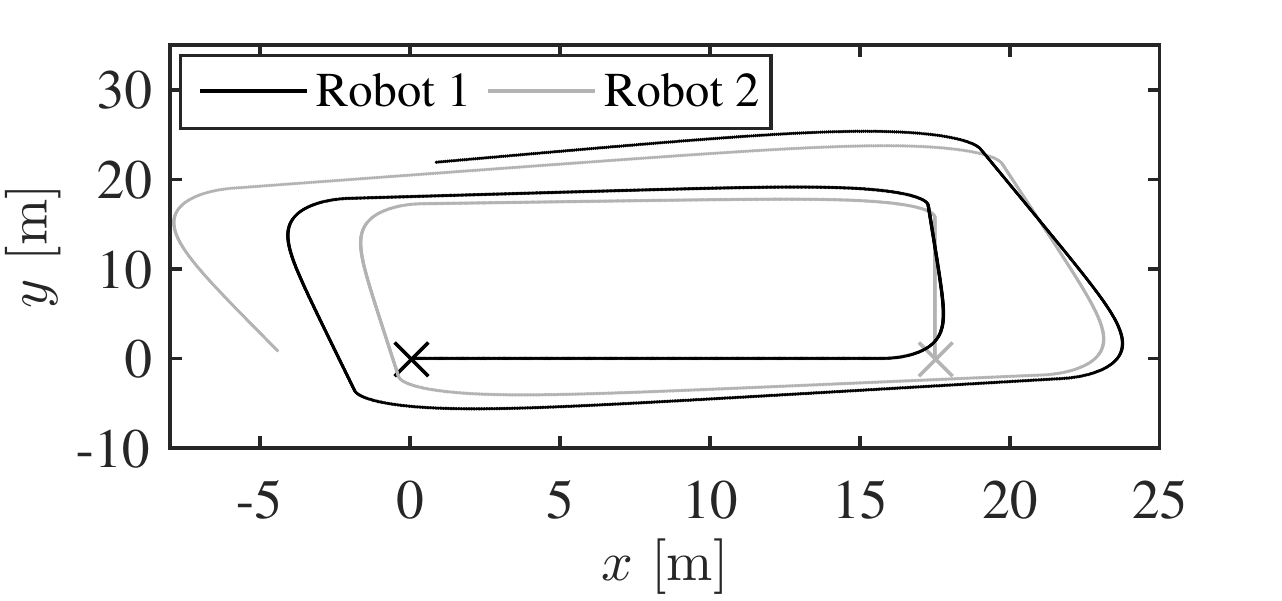}\!\!\!\!\!\!\!
  }
  \subfloat[true trajectory of robots 3 and 4]
  {
    \!  \includegraphics[trim=3 2 4 5,clip,width=0.25\textwidth]{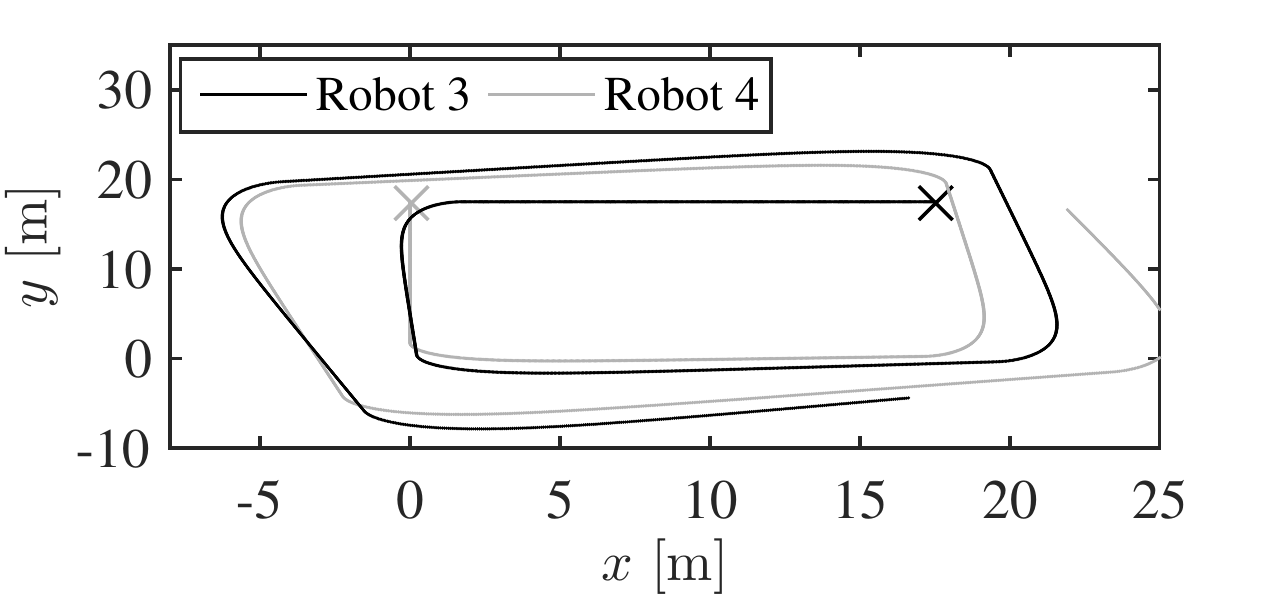}
  }\\ \vspace{-0.08in}
 \subfloat[robot 1]{
 \! \!\!\! \! \!\!\! \includegraphics[trim=3 1 5 3,clip,width=0.24\textwidth]{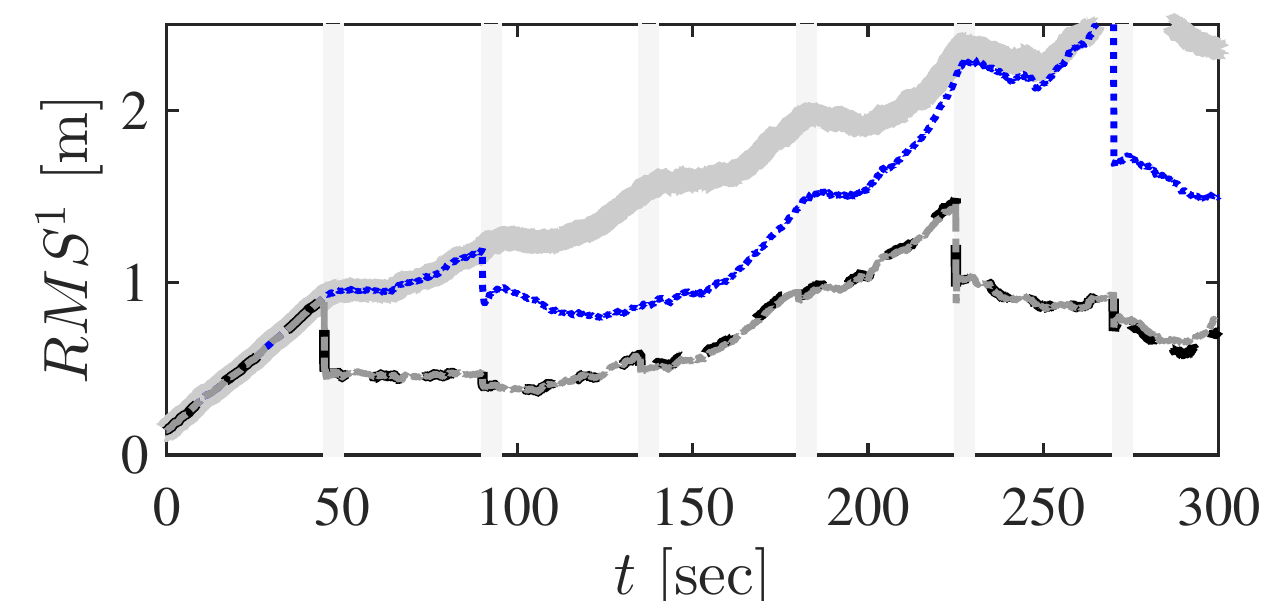}\!\!\!\!
  }
  \subfloat[robot 2]
  {
     \! \includegraphics[trim=3 1 5 3,clip,width=0.24\textwidth]{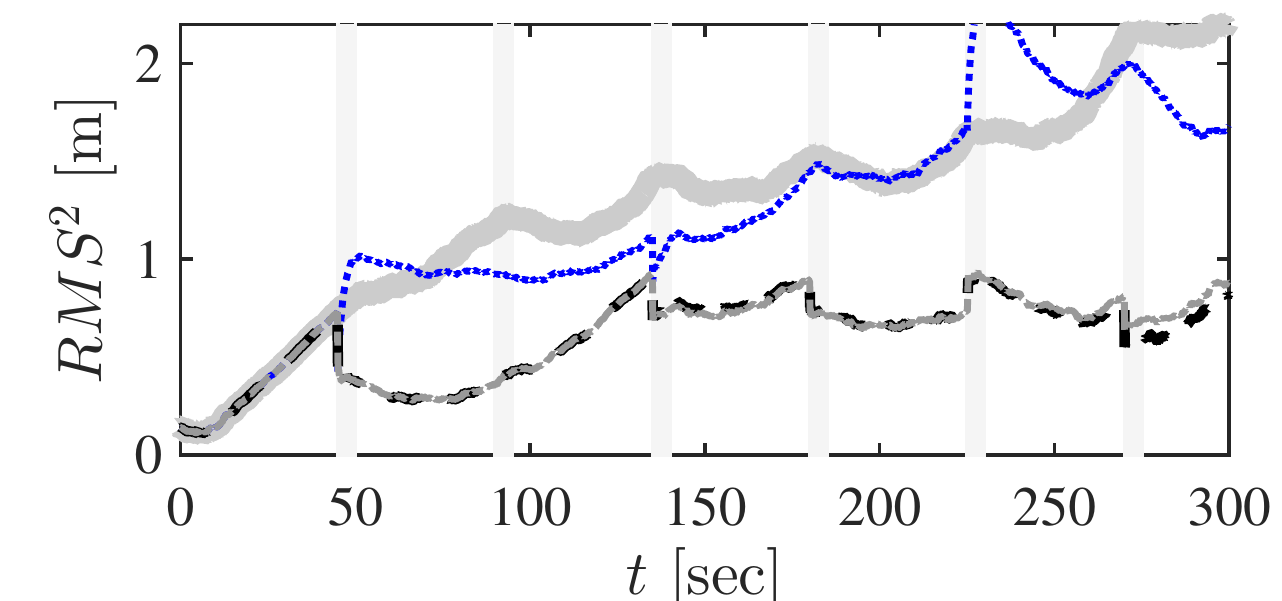}
  }\\ \vspace{-0.06in}
  \subfloat[robot 3]{
  \! \!\!\! \! \!\!\!   \includegraphics[trim=3 1 5 3,clip,width=0.24\textwidth]{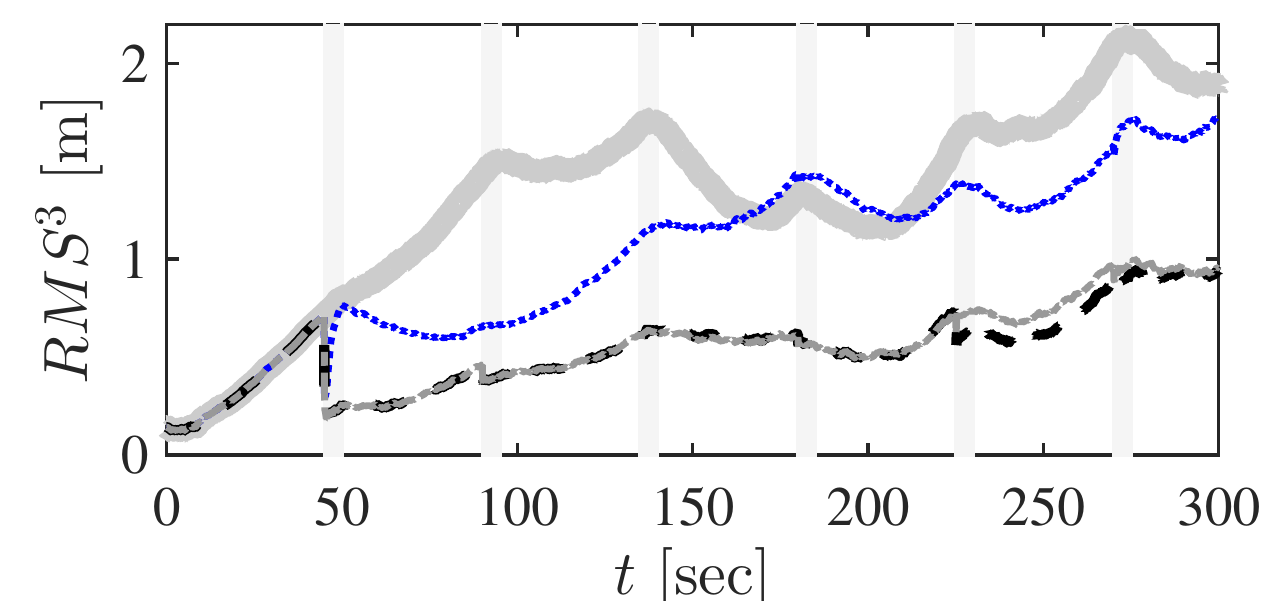}\!\!\!\!
  }
  \subfloat[robot 4]
  {
   \!   \includegraphics[trim=3 1 5 3,clip,width=0.24\textwidth]{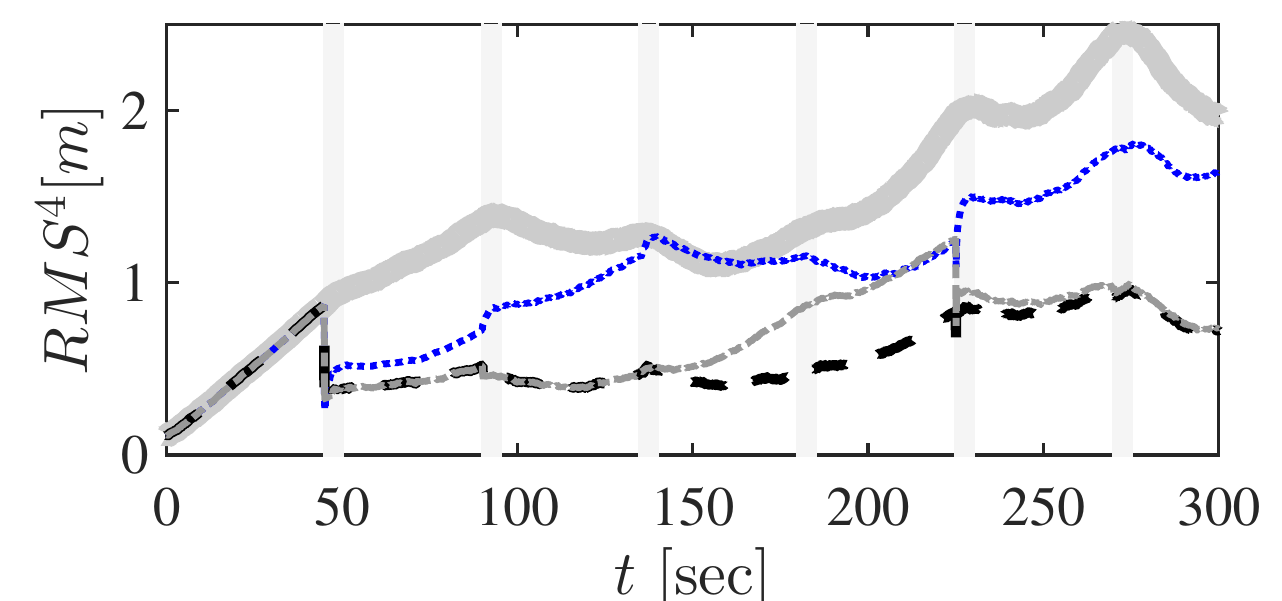}
  }
   \caption{{\small \blue{Simulation results for  position RMS  error for the
    measurement/communication scenario of
    Table~\ref{table::simulation_time} (the orientation RMS error
    behaves similarly and omitted for brevity). In plots (c)-(f),
    ultra thick gray solid line shows the RMS error for dead-recking only; black dashed line and  gray dash dotted line show RMS for \PDsplitEKF CL respectively  in the absence and presence of link failure; and blue dotted line shows the RMS plot for the algorithm of~\cite{HL-FN:13}.
     }  }}\label{fig::simulation}
\end{figure}

\emph{Experimental evaluation}:
we tested the performance of Algorithm~\ref{alg::ouralgpar} and its robustness to message dropouts experimentally, as well.
Our robotic testbed consists of a set of two overhead cameras, a
computer workstation, and $4$ TurtleBot robots (see
Figure~\ref{fig::foto}). This testbed operates under Robot Operating
System (ROS). The overhead cameras, with the help of the set of AR
tags and the ArUco image processing library~\cite{aruco}, are used to
track the motion of the robots and generate a reference trajectory to
evaluate the performance of the CL algorithms. The workstation
serves as the server running a ROS node with the central part of
the~\PDsplitEKF~CL algorithm.
 Each robot has a ROS node that
includes programs to propagate the local filter
equations~\eqref{eq::propag_central_Expanded-a}, \eqref{eq::propag_central_Expanded-b} and~\eqref{eq::Phi} using wheel-encoder measurements and
relative-pose measurements from other robots using the onboard
Kinect camera unit. To take relative-pose measurements, the Kinect
camera also uses a set of AR tags and the ArUco image processing
library. The robots communicate with the workstation via WiFi. The AR
tags are placed on top of the TurtleBot's rack and are arranged on a
cube to provide tags in every horizontal and in top directions. The
accuracy of the visual tag measurements is set to $0.03$ meter for
position and to $6$ degree for orientation. For the propagation stage
of every robot, the local filters of the robots apply the velocity
measurement of their wheel encoders and account the noise with $50\%v$
standard deviation.

\begin{figure}[t!]
  \unitlength=0.5in
  \psfrag*{x}[][cc][1.1]{\renewcommand{\arraystretch}{0.5}\begin{tabular}{c}$\ln(|x^i-\mathsf{x}^\star|)$\\~\end{tabular}}
  \psfrag*{A}[][cc][1.1]{\renewcommand{\arraystretch}{0.5}\begin{tabular}{c}Agents\\~\end{tabular}}
  \psfrag*{t}[][cc][1.1]{\renewcommand{\arraystretch}{2}\begin{tabular}{c}$t$\end{tabular}}
  \centering 
 \!\!\! \subfloat[The robotic testbed]{
   	\includegraphics[width=0.45\linewidth]{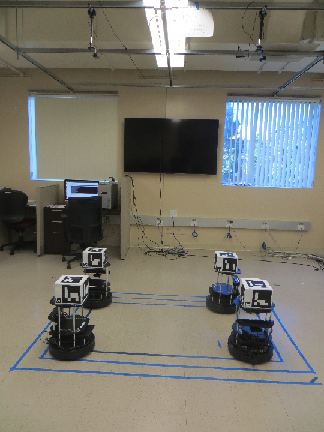}
  }
  \subfloat[Turtlebot with AR tag]
  {
	\includegraphics[width=0.4\linewidth]{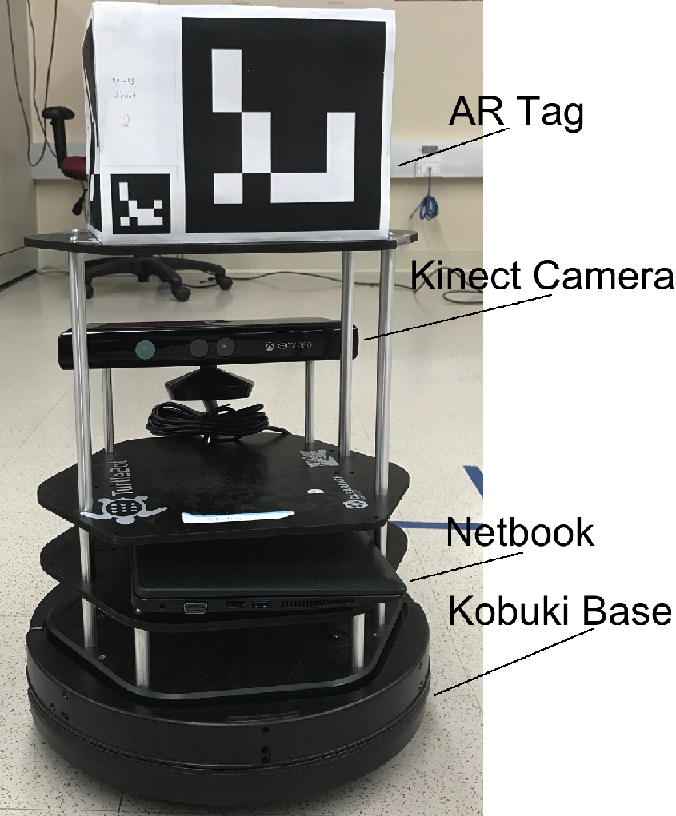}
  }
 	\caption{{\small Setup for the multi-robot test scenario showing the four TurtleBot robots. Every agent features a cube with tags that enable both the Kinect and the overhead camera to take  pose measurements. } }
	\label{fig::foto}
\end{figure}

The robots move in a $2$m $\times$ $3$m area, which is the active
vision zone of our overhead camera system.  The robots move
simultaneously in a counter clock-wise direction along a square
helical path shown in Figure~\ref{fig::foto} and
Figure~\ref{fig::experiment}.  Starting each at one of the four inner
corners of this helical path, marked with large green crosses on
Figure~\ref{fig::experiment}, the robots are programmed to arrive at the next corner ahead of them at
the same time. Along the edge of the track the robots use their wheel
encoder measurements to propagate their motion model while, at the
corners, discrete relative-measurement sequences are executed to
update the local-pose estimates of the robots according to
Algorithm~\ref{alg::ouralgpar}.  In our experiment, the
relative-measurement scenario is for the robot at region $1$ to take
relative measurement from the robot at region $2$, and the robot at
region $2$ to take relative measurement from the robot at region
$3$. The testbed works under perfect communication but we emulate
message dropouts as described below. In our experiment, we execute
the following four estimation filters simultaneously: (a) an overhead
camera tracking to generate the reference trajectory; (b) a
propagation-only filter to demonstrate the accuracy of position
estimates without relative measurements; (c) an execution of the CL
Algorithm~\ref{alg::ouralgpar} under a perfect communication scenario;
(d) an execution of the CL Algorithm~\ref{alg::ouralgpar} under a
measurement-dropout scenario. Note here that each of the CL filters
(c) and (d) has its own corresponding server node on the workstation.
Figure~\ref{fig::experiment} depicts the result of one of our
experiments.  In this experiment, to emulate the message dropout, we
partition our area as shown in Figure~\ref{fig::experiment} into four
regions and designate one of the areas, highlighted in gray, as the
message-dropout zone.  In the implementation that executes CL
Algorithm~\ref{alg::ouralgpar} under the message-dropout scenario
(CL filter (d)), the robot passing through the gray zone does not
implement the update-message it receives from the server.  In
Figure~\ref{fig::experiment}, the trajectory generated by the overhead
camera (the curve indicated by the black crosses) serves as our
reference trajectory. As seen, as times goes by the
position estimate generated by propagating the pose equations using
the wheel encoder measurements (the trajectory depicted by the dotted
curve) has large estimation error.  In Figure~\ref{fig::experiment},
the location estimate of the robots via the CL
Algorithm~\ref{alg::ouralgpar} under perfect-communication and
message-dropout scenarios are depicted, respectively by the solid red
curve and the blue dashed curve. As we can see, whenever a relative
measurement is obtained, the CL algorithms improve the location
accuracy of the robots. Of particular interest is the effect of CL
algorithm on the position accuracy of robots when they pass through
region $4$ (the shaded region on Figure~\ref{fig::experiment}). In our
scenario described above, no relative measurement is taken by or from
the robot in region $4$. However, because of maintained past
correlations among the robots through the server, in the case of the
perfect-communication scenario the robot in region $4$ still
benefits from the relative measurement updates generated by
measurements taken by other robots. 
  Of
course, in the message-dropout scenario (see the blue dashed line
trajectories) such benefit is lost because the robot in region $4$
does not receive the update message from the server. However, the
trajectories show the robustness of Algorithm~\ref{alg::ouralgpar} to
message dropout, i.e., the robots that receive the update message from
the server continue to improve their localization accuracy while the
robot in region $4$ is momentarily deprived from such
benefit. However, as soon as the latter reconnects and
receives an update message, its accuracy improves again.
\begin{figure}[t!]
  \unitlength=0.5in
  \psfrag*{x}[][cc][1.1]{\renewcommand{\arraystretch}{0.5}\begin{tabular}{c}$\ln(|x^i-\mathsf{x}^\star|)$\\~\end{tabular}}
  \psfrag*{A}[][cc][1.1]{\renewcommand{\arraystretch}{0.5}\begin{tabular}{c}Agents\\~\end{tabular}}
  \psfrag*{t}[][cc][1.1]{\renewcommand{\arraystretch}{2}\begin{tabular}{c}$t$\end{tabular}}
  \centering 
\subfloat[robot 1]{
   \includegraphics[trim=0 0 0 15,clip,height=1.55in]{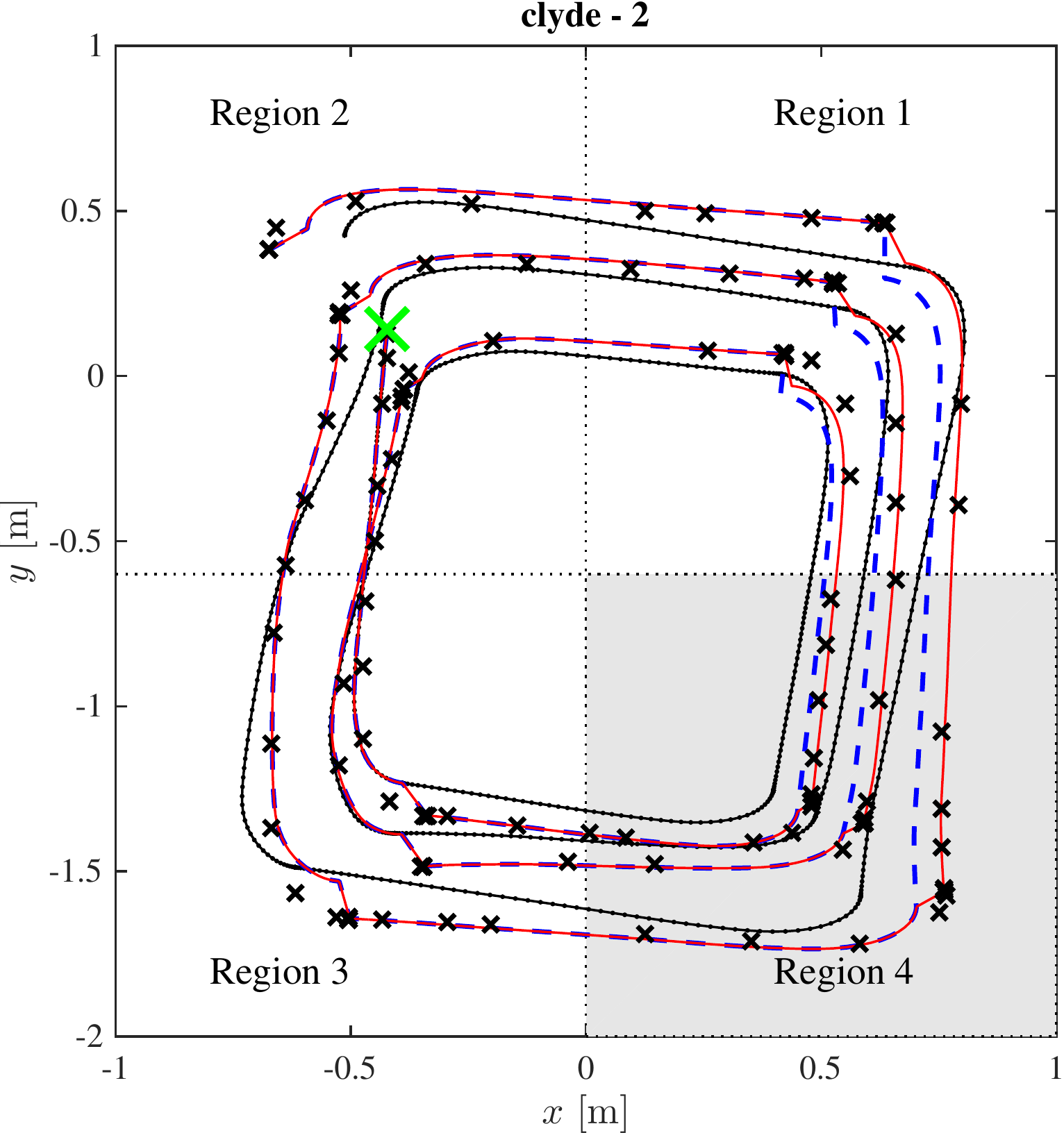}
  }
  \subfloat[robot 2]
  {
    \includegraphics[trim=0 0 0 15,clip,height=1.55in]{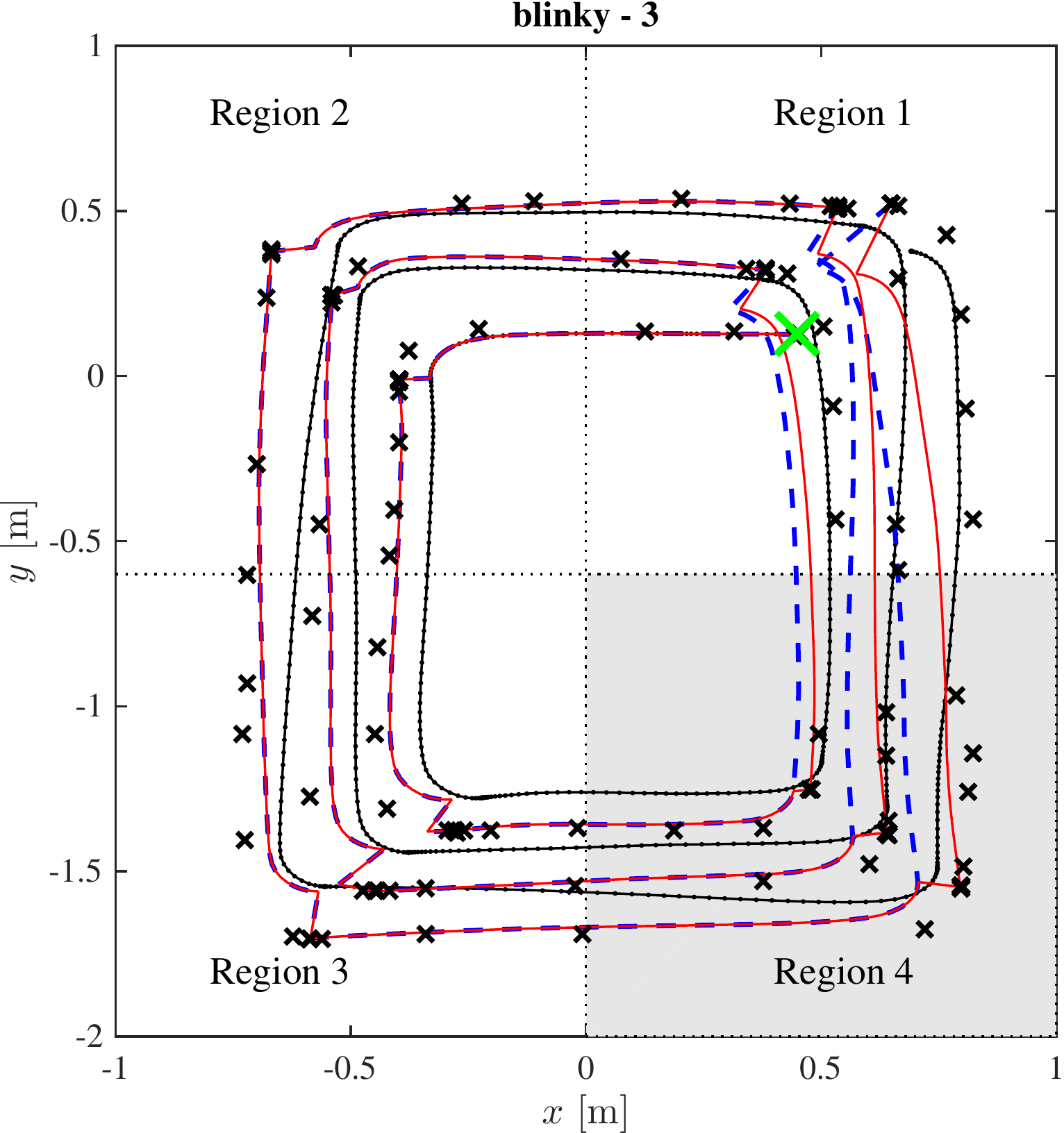}
  }\\
  \subfloat[robot 3]{
    \includegraphics[trim=0 0 0 15,clip,height=1.55in]{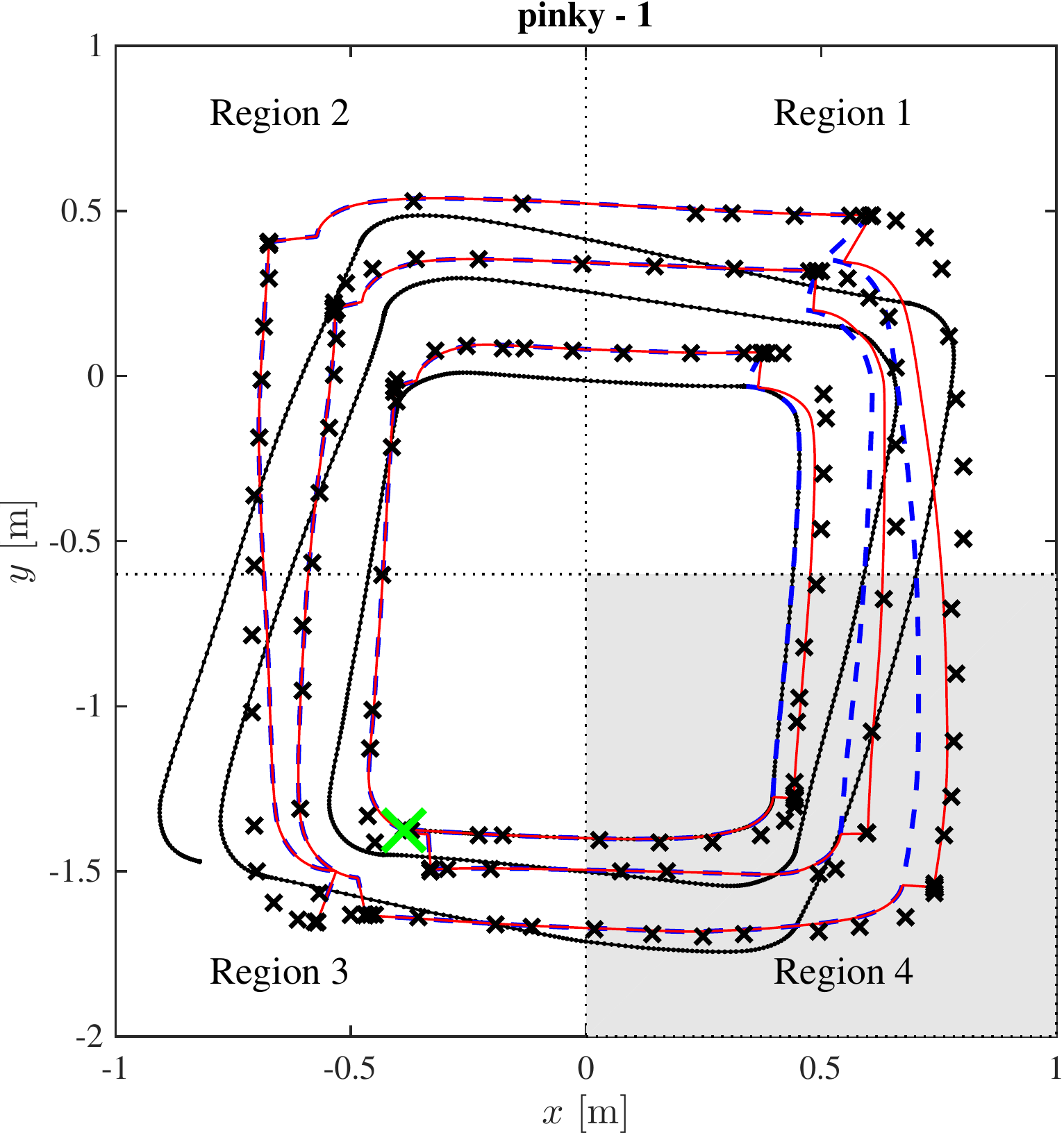}
  }
  \subfloat[robot 4]
  {
    \includegraphics[trim=0 0 0 15,clip,height=1.55in]{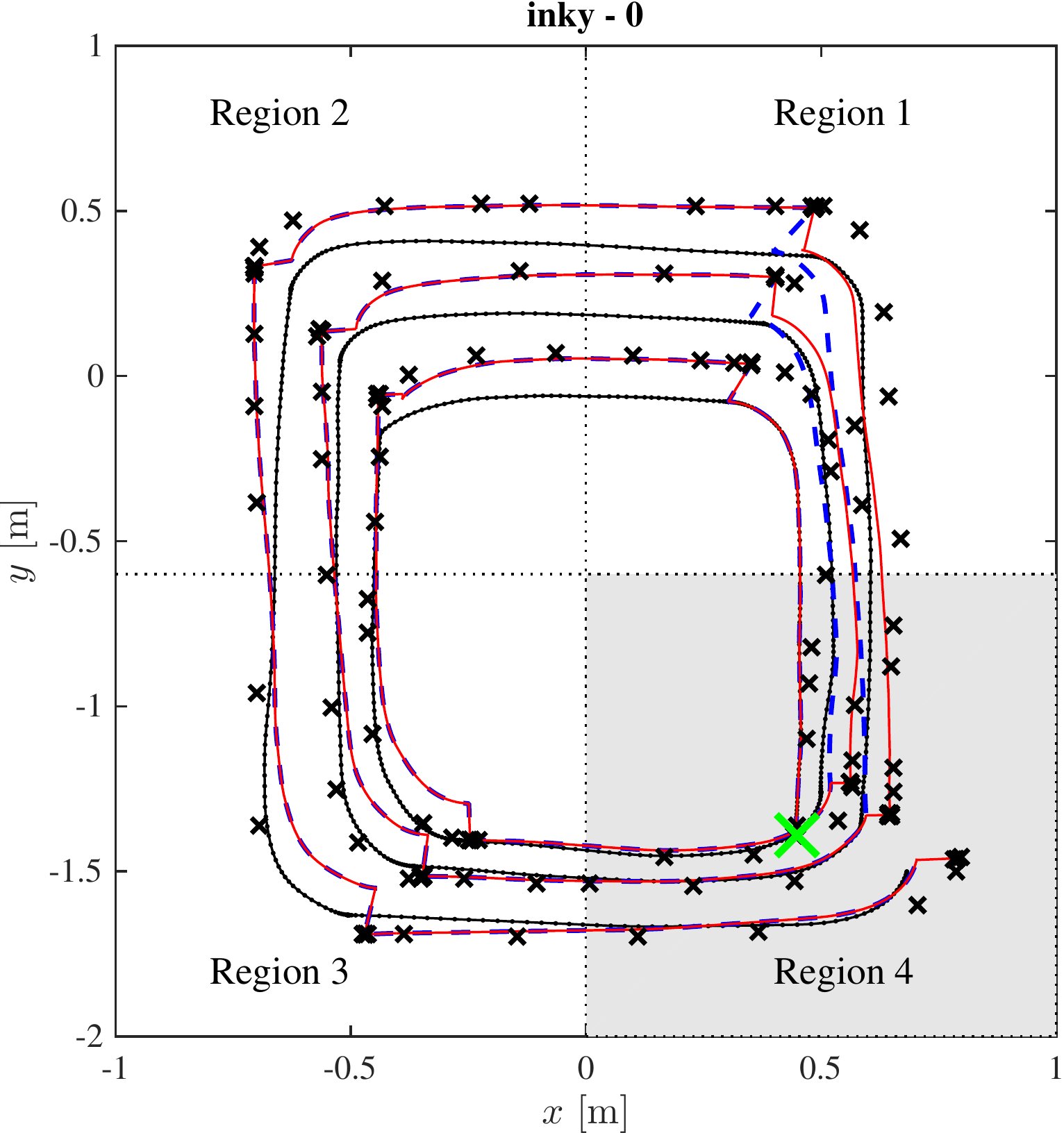}
  }
  \caption{{\small Trajectories of the robots under an experimental
      test generated by $4$ simultaneously running ROS packages, one
      for the overhead camera location tracking (the curve indicated
      by black crosses), one for the propagation only location
      estimate (the black dotted curve), and the other two to obtain
      location estimates by the 
      the~\PDsplitEKF~CL algorithm (Algorithm~\ref{alg::ouralgpar}) under perfect communication (red
      solid curve) and message-dropout (dashed blue curve)
      scenarios. Region $4$ which is highlighted in gray is the area
      where we emulate the message dropout.
}  }\label{fig::experiment}
\end{figure}

\section{Conclusions}\vspace{-0.08in}
\blue{For a team of robots with limited computational, storage and
communication resources, we proposed a server assisted distributed CL algorithm  which under the perfect communication scenarios renders the same 
localization performance as of a joint CL using EKF. In terms of the team size,
this algorithm only requires {$O(1)$} storage and computational cost per
robot and the main computational burden of implementing the EKF for CL
is carried out by the server. }We showed that this algorithm has robustness to occasional communication failure between robots and the server. 
Here, we discarded the measurement of the robots that fail to communicate with the server. 
Our future work involves utilizing these old measurements using
out-of-sequence-measurement update strategies~\cite{YBS-HC-MM:04} when
the communication link is restored between the corresponding robot and
the~server.



\vspace{-0.06in}
\bibliographystyle{ieeetr}%


\appendix[Proof of Theorems~\ref{thm::main} and~\ref{thm::partial_update} ]
\renewcommand{\theequation}{A.\arabic{equation}}
\renewcommand{\thethm}{A.\arabic{thm}}
\renewcommand{\thelem}{A.\arabic{lem}}
\renewcommand{\thedefn}{A.\arabic{defn}}

\begin{proof}[Proof of Theorem~\ref{thm::main}]
   Our proof is based on the
  mathematical induction over $k\in\nonnegativeinteger$.
Let $k=0$. Given~\eqref{eq::intermidate_var} and the defined
  initial conditions,~the right hand side of \eqref{eq::alternate-EKF-equations-a} results in
$\!\vect{\Phi}^i(1)\,\vect{\Pi}_{i,j}(0)\,
  \vect{\Phi}^j(1)\!^\top\!\!\!=\!\vect{F}^i(1)\,\vect{0}_{n^i\times
    n^j}\,\vect{F}^j(1)\!^\top\!\!\!=\!\vect{0}_{n^i\times n^j}
$,
which matches exactly the value~\eqref{eq::propag_central_Expanded-c} gives for $\vect{P}_{i,j}^{\prpg}(1)$. Next, we validate~\eqref{eq::alternative-EKF-update-xi-Pi} at $k=0$. When there is no relative
measurement at the first step, because of~\eqref{eq::barD-no-meas}, \eqref{eq::alternate-EKF-equations-b} and \eqref{eq::alternate-EKF-equations-c} give respectively $\Hvect{x}^{i\updt}(1)=\Hvect{x}^{i\prpg}(1)$ and $\vect{P}^{i\updt}(1)=\vect{P}^{i\prpg}(1)$ which match exactly what~\eqref{eq::RobotCovarUpdate-a} and \eqref{eq::RobotCovarUpdate-b} provide. Given~\eqref{eq::barD-i}-\eqref{eq::barD-b} and \eqref{eq::Pi}, the right hand side of~\eqref{eq::alternate-EKF-equations-d} reads as
$\vect{\Phi}^i(1)\,\vect{\Pi}_{i,j}(1)
\,\vect{\Phi}^j(1)^\top=\vect{\Phi}^i(1)\,\vect{\Pi}_{i,j}(0)
\,\vect{\Phi}^j(1)^\top=\vect{0}_{n^i\times n^j}$ which matches exactly the value 
~\eqref{eq::RobotCovarUpdate-c} gives for $\vect{P}_{i,j}^{\updt}$. 
On the other hand, when there
is a relative measurement $a\xrightarrow{1}b$, validity of~\eqref{eq::alternative-EKF-update-xi-Pi} follows from followings.
Using~\eqref{eq::barD-i}-\eqref{eq::barD-b}, we obtain
$\vect{\Gamma}_i(1)=(\vect{\Pi}_{i,b}(0){\vect{\Phi}^b(1)}^\top\Tvect{H}_{b}^\top\!+
  \!\vect{\Pi}_{i,a}(0){\vect{\Phi}^a(1)}^\top\Tvect{H}_{a}^\top)=\vect{0},~ i\in\VV\backslash\{a,b\}$,
$\vect{\Phi}^a(1)\vect{\Gamma}_{a}(1)=   \vect{F}^a(1)\vect{F}^a(1)^{-1}\vect{P}^{a\prpg}(1)\Tvect{H}_{a}(1)^\top
  \,{\vect{S}_{a,b}}(1)^{-\frac{1}{2}}=\vect{K}_a(1){\vect{S}_{a,b}}(1)^{\frac{1}{2}}$ and $\vect{\Phi}^b(1)\vect{\Gamma}_{b}(1)=
  \vect{F}^{b}(1) \vect{F}^{b}(1)^{-1}\vect{P}^{b\prpg}(1)\Tvect{H}_{b}(1)^\top\!\,{\vect{S}_{a,b}}(1)^{-\frac{1}{2}}=\vect{K}_b(1)\vect{S}_{a,b}^{\frac{1}{2}}.$ Moreover, $\vect{\Pi}_{i,j}(1)=\vect{\Pi}_{i,j}(0)+\vect{\Gamma}_i(1)\,\vect{\Gamma}_j(1)=\vect{0}$,~ $i\in\VV\backslash\{a,b\}$, $j\in\VV\backslash\{i,a,b\}$. Here, we used $\vect{\Phi}^{a}(1)=\vect{F}^{a}(1)$, $\vect{\Phi}^{b}(1)=\vect{F}^{b}(1)$, $\vect{\Pi}_{i,j}(0)=\vect{0}$ and $\vect{P}^{\prpg}_{i,j}(0)=\vect{0}$ for $i\in\VV$, $j\in\VV\backslash\{i\}$. Therefore,~\eqref{eq::alternate-EKF-equations-d} gives $\vect{P}_{i,j}^{\updt}(1)=\vect{0}$, $i\in\VV\setminus\{a,b\},
  ~j\in\VV\backslash\{i,a,b\}$, and 
  \begin{align*}
  &\vect{P}_{a,b}^{\updt}(1)=\!\vect{P}_{b,a}^{\updt}(1)^\top\!=\!\vect{\Phi}^a(1)(\vect{\Pi}_{a,b}(0)\!-\!\vect{\Gamma}_a(1)\,
  \vect{\Gamma}_b(1)^\top)  \vect{\Phi}^b(1)^\top\\
& ~=-\vect{\Phi}^a(1)\vect{\Gamma}_a(1)\,
  \vect{\Gamma}_b(1)^\top \vect{\Phi}^b(1)^\top=-\vect{K}_a(1)\vect{S}_{a,b}(1)\vect{K}_b(1)^\top,
  \end{align*}
 which exactly matches~\eqref{eq::RobotCovarUpdate-c} as shown below
 (recall~\eqref{eq::K_gain_robotwise}). First note that,
 \eqref{eq::RobotCovarUpdate-c} reduces to $
   \vect{P}_{i,j}^{\updt}(1) =\vect{0}$ for $i\in\VV\setminus\{a,b\}$ and $j\in\VV\backslash\{i,a,b\}$. We also obtain $
   \vect{P}_{a,b}^{\updt}(1)=\vect{P}_{b,a}^{\updt}(1)^\top=\vect{P}_{a,b}^{\prpg}
   \!(1)-\vect{K}_a(1)\vect{S}_{a,b}(1)\vect{K}_b(1)^\top=-\vect{K}_a(1)\vect{S}_{a,b}(1)\vect{K}_b(1)^\top$.
   
Assume now that the theorem statement holds for $k$. Then at time step
$k+1$, we have $\vect{\Phi}^i(k+1)\vect{\Pi}_{i,j}(k)\vect{\Phi}^j(k+1)^\top=\vect{F}^i(k+1)\vect{\Phi}^i(k)\vect{\Pi}_{i,j}(k)\vect{\Phi}^j(k)^\top\vect{F}^j(k+1)^\top=\vect{F}^i(k+1)\vect{P}_{i,j}^{\updt}(k)\vect{F}^j(k+1)^\top= \vect{P}_{i,j}^{\prpg}(k+1)$, 
 which confirms validity of \eqref{eq::alternate-EKF-equations-a} at
 $k+1$. Next, we show~\eqref{eq::alternate-EKF-equations-d} is correct. When there
 is no relative measurement at $k+1$, using~\eqref{eq::barD-i}-\eqref{eq::barD-b} and~\eqref{eq::Pi} 
 we can write  $\vect{\Phi}^i(k+1)\vect{\Pi}_{i,j}(k+1)\vect{\Phi}^j(k+1)^\top
=\vect{\Phi}^i(k+1)\vect{\Pi}_{i,j}(k)\vect{\Phi}^j(k+1)^\top$, which confirms the correct outcome of 
$\vect{P}_{i,j}^{\updt}(k+1)
   =\vect{P}_{i,j}^{\prpg}(k+1)$ holds at $k+1$.
Next, we
 evaluate~\eqref{eq::alternate-EKF-equations-d} when robot $a$ takes a
 relative measurement from robot $b$ at $k+1$.  First, notice that we can always write
 {\begin{align}\label{eq::K_alt}
     \vect{K}_i(k+1)=\vect{\Phi}^i(k+1)\vect{\Gamma}_i(k+1)\vect{S}_{a,b}^{-\frac{1}{2}},\quad
     i\in\VV,
\end{align}}
becuase
\begin{itemize}
\item for $i\in\VV\backslash\{a,b\}$ (recall~\eqref{eq::barD-i},
  \eqref{eq::alternate-EKF-equations-a},) we have
     $\vect{\Phi}^i(k\!+\!1)\vect{\Gamma}_i(k\!+\!1)\vect{S}_{a,b}^{-\frac{1}{2}}\!=\!
      \vect{\Phi}^i(k\!+\!1)(\vect{\Pi}_{i,b}(k){\vect{\Phi}^b(k\!+\!1)}\!^\top\Tvect{H}_{b}\!^\top\!
      \!\!+\!\vect{\Pi}_{i,a}(k){\vect{\Phi}^a(k\!+\!1)}^\top\Tvect{H}_{a}\!^\top){\vect{S}_{a,b}}\!\!^{-1}=(\vect{P}_{i,b}^{\prpg}(k\!+\!1)\Tvect{H}_{b}^\top+\vect{P}_{i,a}^{\prpg}(k\!+\!1)
      \Tvect{H}_{a}^\top){\vect{S}_{a,b}}\!\!^{-1}\!=\!\vect{K}_i(k\!+\!1)$;
\item for $i=a$ (recall~\eqref{eq::barD-a},
  \eqref{eq::alternate-EKF-equations-a},) we have $\vect{\Phi}^i(k+1)\vect{\Gamma}_i(k+1)\vect{S}_{a,b}^{-\frac{1}{2}}=
    \vect{\Phi}^a(k+1)\big(\vect{\Pi}_{a,b}(k){\vect{\Phi}^b\!(k\!+\!1)}\!^\top\Tvect{H}_{b}^\top\!
    \!\!+\!\vect{\Phi}^a\!(k\!+\!1)^{-1}\vect{P}^{a\prpg}\!(k\!+\!1)\Tvect{H}_{a}\!\!^\top\big){\vect{S}_{a,b}}\!\!^{-1}=\!(\vect{P}_{a,b}^{\prpg}(k+1)\Tvect{H}_{b}\!^\top\!+\!\vect{P}^{a\prpg}(k\!+\!1)
    \Tvect{H}_{a}^\top){\vect{S}_{a,b}}\!\!^{-1}\!\!=\!\vect{K}_a(k\!+\!1)$;
\item for $i=b$ (recall~\eqref{eq::barD-b},
  \eqref{eq::alternate-EKF-equations-a},) we have $\vect{\Phi}^i(k+1)\vect{\Gamma}_i(k+1)\vect{S}_{a,b}^{-\frac{1}{2}}
  = \vect{\Phi}^b(k+1) \big(\vect{\Phi}^{b}(k+1)^{-1}\vect{P}^{b\prpg}(k+1)
  \Tvect{H}_{b}^\top+\!\vect{\Pi}_{b,a}(k)
  {\vect{\Phi}^a(k+1)}^\top\Tvect{H}_{a}^\top\big)\,{\vect{S}_{a,b}}\!\!^{-1}=(\vect{P}^{b\prpg}(k+1)\Tvect{H}_{b}^\top\!+\!\vect{P}^{\prpg}_{b,a}(k)
  \Tvect{H}_{a}^\top)\,{\vect{S}_{a,b}}\!\!^{-1}\!=\!\vect{K}_b(k+1)$.
\end{itemize}
Therefore, by recalling~\eqref{eq::barD-i}-\eqref{eq::barD-b} and~\eqref{eq::Pi}, we can write $
  \vect{\Phi}^i(k+1)\vect{\Pi}_{i,j}(k+1)
  \vect{\Phi}^j(k+1)^\top=\vect{\Phi}^i(k+1)\vect{\Pi}_{i,j}(k)\vect{\Phi}^j(k+1)^\top-\vect{\Phi}^i(k+1)\vect{\Gamma}_i(k+1)\, \vect{\Gamma}_j(k+1)^\top\vect{\Phi}^j(k+1)^\top=\vect{P}_{i,j}^{\prpg}(k+1)-\big(\vect{\Phi}^i(k+1)\vect{\Gamma}_i(k+1)\,
  \vect{S}_{a,b}^{-\frac{1}{2}}\big)\vect{S}_{a,b}\,
  \big(\vect{\Phi}^j(k+1)\vect{\Gamma}_j(k+1)
  \vect{S}_{a,b}^{-\frac{1}{2}}\big)^\top=\vect{P}_{i,j}^{\prpg}(k+1)-\vect{K}_i(k+1)\vect{S}_{a,b}
  \vect{K}_j(k+1)^\top= \vect{P}_{i,j}^{\updt}(k+1)$,
 which confirms validity of \eqref{eq::alternate-EKF-equations-d} at
 $k+1$ when robot $a$ takes relative measurement from robot $b$. This
 completes the proof of validity
 of~\eqref{eq::alternate-EKF-equations-d} for all
 $k\in\nonnegativeinteger$. Subsequently,~\eqref{eq::alternate-EKF-equations-b} and~\eqref{eq::alternate-EKF-equations-c}
 follow, in a straightforward manner, from~\eqref{eq::K_alt} now
 being valid for all $k\in\nonnegativeinteger$.
\end{proof}

\begin{proof}[Proof of Theorem~\ref{thm::partial_update}]         
  We can obtain Kalman gain $\vect{K}_{1:m}$ that minimizes  $\text{Trace}(\vect{P}^{\updt}(k+1))$ from $\partial \text{Trace}(\vect{P}^{\updt}(k+1))
  /\partial \vect{K}_{1:m}=\vect{0}$.  Let
  {$\Hvect{x}_{1:m}^{\updt}=(\Hvect{x}^{1\updt},\cdots,\Hvect{x}^{m\updt})$},
  {$\Hvect{x}_{m+1:N}^{\updt}=(\Hvect{x}^{m+1\updt},\cdots,\Hvect{x}^{N\updt})$}. Next,
  we obtain
  $\text{Trace}(\vect{P}^{\updt}(k+1))$. Given~\eqref{eq::update-cent-miss},
we have\begin{align*}
&
\begin{bmatrix}\begin{smallmatrix}
  \vect{x}_{1:m}(k+1)-\vect{x}_{1:m}^{\updt}(k+1)\\
  \vect{x}_{m+1:N}(k+1)-\vect{x}_{m+1:N}^{\updt}(k+1)\end{smallmatrix}
\end{bmatrix}\approx \begin{bmatrix}\begin{smallmatrix}
\vect{e}_{1:m}^{\updt}(k+1)\\
\vect{e}_{m+1:N}^{\updt}(k+1)\end{smallmatrix}
\end{bmatrix}=\!\nonumber\\
  &
\begin{bmatrix}\begin{smallmatrix}(\vect{I}_m-\!\vect{K}_{1:m}\Bvect{H})&\vect{0}\\
  \vect{0}&\vect{I}_{N-m}\end{smallmatrix}\end{bmatrix}\!\!\!\begin{bmatrix}\begin{smallmatrix}
  \vect{x}_{1:m}(k\!+\!1)-\vect{x}_{1:m}^{\prpg}(k\!+\!1)\\
  \vect{x}_{m+1:N}(k\!+\!1)-\vect{x}_{m+1:N}^{\prpg}(k\!+\!1)\end{smallmatrix}\end{bmatrix}\\
  &~\qquad\qquad\quad\qquad+\begin{bmatrix}\begin{smallmatrix}
  -\vect{K}_{1:m}&\vect{0}\\
  \vect{0}&\vect{0} \end{smallmatrix}\end{bmatrix}\!\!\!  \begin{bmatrix}\begin{smallmatrix}
  \vect{\nu}^a(k+1)\\
  \vect{0}\end{smallmatrix} \end{bmatrix}
\!\!, \end{align*} where 
$\Bvect{H}= \big[\begin{smallmatrix}
  \overset{1}{\vect{0}}~~\overset{\cdots}{\cdots}~~\overset{a}{\Tvect{H}_a}(k+1)~~\overset{a+1}{\vect{0}}~~\overset{\cdots}{\cdots}~~\overset{b}{\Tvect{H}_b}(k+1)~~\overset{b+1}{\vect{0}}~~\overset{\cdots}{\cdots}\overset{m}{\vect{0}} \end{smallmatrix}\big]$.
Recall that
$\vect{P}^{\updt}(k+1)\!=\!\text{E}[\vect{e}^{\updt}(k+1)\vect{e}^{\updt}(k+1)^\top]$
which is equal~to
\begin{align}\label{eq::partial_covariance_update}
  & \vect{P}^{\updt}(k+1)=\! \\
  & \begin{smallmatrix}\left[\begin{array}{c|c}
      \vect{P}^{\updt}_{1:m,1:m}(k+1)&\vect{P}^{\updt}_{1:m,m+1:N}(k+1)\\\hline
      \vect{P}^{\updt}_{1:m,m+1:N}(k+1)^\top&\vect{P}^{\updt}_{m+1:N,m+1:N}(k+1)
    \end{array}\right]\!\end{smallmatrix}=\nonumber\\
    &\!\begin{bmatrix}\begin{smallmatrix}
    \vect{K}_{1:m}\vect{R}_a\vect{K}_{1:m}^\top&\vect{0}_{m\times
      (N-m)}\\
    \vect{0}_{(N-m)\times m}&\vect{0}_{(N-m)\times(N-m)}\end{smallmatrix}\end{bmatrix}\!+\!\begin{bmatrix}\begin{smallmatrix}(\vect{I}_m\!-\!\vect{K}_{1:m}\Bvect{H})\!\!\!\!\!&\vect{0}\\
    \vect{0}&\!\!\!\vect{I}_{N-m}\end{smallmatrix}\end{bmatrix}\nonumber\\
  &\!\!\begin{smallmatrix}\times\left[
    \begin{array}{c|c}\vect{P}^{\prpg}_{1:m,1:m}(k+1)&\vect{P}^{\prpg}_{1:m,m+1:N}(k+1)\\\hline
      \vect{P}^{\prpg}_{m+1:N,1:m}(k+1)&\vect{P}^{\prpg}_{m+1:N,m+1:N}(k+1)
    \end{array}\right]\end{smallmatrix}\times\nonumber\\&\!\!\begin{bmatrix}\begin{smallmatrix}
    (\vect{I}_m-\vect{K}_{1:m}\Bvect{H})^\top&\vect{0}\\
    \vect{0}&\vect{I}_{N-m}\end{smallmatrix}\end{bmatrix}\!\!.\nonumber
\end{align}
Then, we have
$ \tr{\vect{P}^{\updt}(k+1)}=\tr{\vect{P}^{\prpg}_{1:m,1:m}(k+1)}\!-2\,
  \tr{\vect{K}_{1:m}\Bvect{H}\vect{P}^{\prpg}_{1:m,1:m}(k+1)} +
  \tr{\vect{K}_{1:m}(\vect{R}_a+\Bvect{H}\vect{P}^{\prpg}_{1:m,1:m}(k+1)
    \Bvect{H}^\top)\vect{K}_{1:m}^\top}+
    \tr{\vect{P}^{\prpg}_{m+1:N,m+1:N}(k+1)}$.
As a result, we have $
  \partial \tr{\vect{P}^{\updt}(k+1)} /\partial \vect{K}_{1:m}= -2
  \,\vect{P}^{\prpg}_{1:m,1:m}(k+1)\Bvect{H}^\top+
  2\,(\vect{R}_a+\Bvect{H}
  \vect{P}^{\prpg}_{1:m,1:m}(k+1)\Bvect{H}^\top)\vect{K}_{1:m}^\top
  =-2
  \,\vect{P}^{\prpg}_{1:m,1:m}(k+1)\Bvect{H}^\top+2\,\vect{S}_{a,b}\vect{K}_{1:m}^\top$.
Therefore, the gain {$\vect{K}_{1:m}$} that minimizes $\tr{\vect{P}^{\updt}_{1:m}(k+1)}$ is {$\vect{K}_{1:m}=
  \Bvect{H}\vect{P}^{\prpg}_{1:m,1:m}(k+1)\vect{S}_{a,b}^{-1}$}, which
equivalently expands in robot-wise components to give
us~\eqref{eq::gain_partial_update}. For the covariance update,
from~\eqref{eq::partial_covariance_update}, we obtain 
\begin{subequations}\label{eq::covar-miss-up-manimp}
\begin{align}
  &\vect{P}^{\updt}_{1:m,1:m}(k+1)=\,(\vect{I}_m-\vect{K}_{1:m}\Bvect{H})
 \times\label{eq::covar-miss-up-manimp-11}\\
  &\qquad \vect{P}^{\prpg}_{1:m,1:m}(k+1)(\vect{I}_m-\vect{K}_{1:m}\Bvect{H})^\top+\vect{K}_{1:m}\vect{R}_a\vect{K}_{1:m}^\top\nonumber\\
  &\quad\quad\quad\quad\quad\quad~~\,=\,
  \vect{P}^{\prpg}_{1:m,1:m}(k+1)-\vect{K}_{1:m}\vect{S}_{a,b}\vect{K}_{1:m}^\top,\nonumber\\
  &\vect{P}^{\updt}_{m+1:N,m+1:N}(k+1)=\,\vect{P}^{\prpg}_{m+1:N,m+1:N}(k+1),\label{eq::covar-miss-up-manimp-22}\\
  &\vect{P}^{\updt}_{1:m,m+1:N}(k\!+\!1)\!=\!(\vect{I}_m\!-\!\vect{K}_{1:m}
  \Bvect{H})\times\\
 & \vect{P}^{\prpg}_{1:m,m+1:N}(k\!+\!1)\!=\!
  \Big(\vect{I}_m\!\!-\!
  \begin{bmatrix}\begin{smallmatrix}
    \vect{K}_{1}\\\vdots\\\vect{K}_{m}
  \end{smallmatrix}\end{bmatrix}\!\!\!
\Big[
    \overset{1}{\vect{0}}~\,\overset{\cdots}{\cdots}~\,\overset{a}{\Tvect{H}_a}(k)~\,\overset{a+1}{\vect{0}}~\nonumber\\
    &\qquad\quad\,\overset{\cdots}{\cdots}~\,\overset{b}{\Tvect{H}_b}(k)~\,\overset{b+1}{\vect{0}}~\,\overset{\cdots}{\cdots}\overset{m}{\vect{0}}
\Big]
\!\Big)\vect{P}^{\prpg}_{1:m,m+1:N}(k+1)=\nonumber\\
&
\Big(\!\vect{I}_m\!-\!\begin{bmatrix}\begin{smallmatrix}\vect{K}_{1}
  \vect{S}_{a,b}\vect{S}_{a,b}^{-1}\\\vdots\\\vect{K}_{m}\vect{S}_{a,b}\vect{S}_{a,b}^{-1}
\end{smallmatrix}\end{bmatrix}\Big[\overset{1}{\vect{0}}~\,\overset{\cdots}{\cdots}~\,\overset{a}{\Tvect{H}_a}(k)\nonumber\\
&
\qquad\quad~\,\overset{a+1}{\vect{0}}~\,\overset{\cdots}{\cdots}~\,\overset{b}{\Tvect{H}_b}(k)~\,\overset{b+1}{\vect{0}}~\,\overset{\cdots}{\cdots}\overset{m}{\vect{0}}
\Big]\!\Big)\vect{P}^{\prpg}_{1:m,m+1:N}(k+1),\nonumber
\end{align}
\end{subequations}
where
\begin{align*}
\vect{P}^{\prpg}_{1:m,m+1:N}(k\!+\!1)\!=\!
\begin{bmatrix}\begin{smallmatrix}\vect{P}^{\prpg}_{1,m+1}(k\!+\!1)\!&\cdots&\!\vect{P}^{\prpg}_{1,N}(k\!+\!1)\\
\vdots&\cdots&\vdots\\
\vect{P}^{\prpg}_{m,m+1}(k\!+\!1)\!&\cdots&\!\vect{P}^{\prpg}_{m,N}(k\!+\!1)
\end{smallmatrix} \end{bmatrix}
\end{align*}
Recalling the definition of the pseudo-gains~\eqref{eq::pseudo-gain}, then~\eqref{eq::covar-miss-up-manimp} results in~\eqref{eq::robot-covar-updt-miss} and~\eqref{eq::robot-cross-covar-updt-miss}.
\end{proof}

\clearpage
\appendix[Sequential updating for multiple measurements]
\renewcommand{\theequation}{B.\arabic{equation}}
\renewcommand{\thethm}{B.\arabic{thm}}
\renewcommand{\thelem}{B.\arabic{lem}}
\renewcommand{\thedefn}{B.\arabic{defn}}
For multiple synchronized measurements, we use the
sequential updating procedure.  Let $\VV_{\text{A}}(k)$ denote the set of the robots that have made an exteroceptive measurement at time $k$, $\VV_\text{\text{B}}^i(k)$ denote the landmark robots of robot $i\in\VV_{\text{A}}(k)$, and $\VV_{\text{A,B}}(k)$ represent the set of all landmark robots and the robots that have taken relative measurements. Then the total number of relative measurements is   $n_s=\sum\nolimits_{i=1}^{|\VV_{\text{A}}(k)|}|\VV_\text{\text{B}}^i(k)|$.
    Recall that in sequential updating, the measurements are processed one by one, starting with using the first measurement to update the predicted estimate and error covaraince matrix, and proceeding with next measurement to update the current updated state estimate and error measurements. By straightforward substitution, the sequential updating procedure in \splitEKF~CL variables, starting with 
     $\Hvect{x}^{i\updt}(k+1,0)=\Hvect{x}^{i\prpg}(k+1)$,
  $\vect{P}^{i\updt}(k+1,0)=\vect{P}^{i\prpg}(k+1)$, $i\in\VV$, and
  $\vect{P}_{i,l}^{\updt}(k+1,0)=\vect{P}_{i,l}^{i\prpg}(k+1)$ for
  $l\in\VV\backslash\{i\}$, reads as  (starting at $j=1$),
 \begin{align*}
    &\text{for}~ a\in\VV_{\text{A}}(k+1),\nonumber
   \\
   &\quad\text{for}~ b\in\VV^a_{\text{B}}(k+1),\quad\quad\quad\nonumber\\
   &\!\Hvect{x}^{i\updt}(k\!+\!1,j)\!=\Hvect{x}^{i\prpg}(k\!+\!1,j-1)+\\
   &\quad\quad\qquad\qquad\qquad\quad\vect{\Phi}^{i}(k\!+\!1)\vect{\Gamma}_i(k\!+\!1,j)\Bvect{r}^{a}(k\!+\!1,j),\nonumber\\
   &\!\vect{P}^{i\updt}(k\!+\!1,j)\!=\vect{P}^{i\prpg}\!(k\!+\!1,j-1)-\\
   &\qquad\quad\quad\quad\vect{\Phi}^{i}(k+1) \vect{\Gamma}_{i}(k+1,j)\vect{\Gamma}_i\!^\top(k+1,j)\vect{\Phi}^i(k+1)\!^\top\!,\nonumber\\
   &\!\vect{P}_{i,l}^{\updt}(k\!+\!1,j)\!=\vect{\Phi}^i(k\!+\!1)\,\vect{\Pi}_{i,l}(k\!+\!1,j-1)\,
   \vect{\Phi}^l(k\!+\!1)\!^\top\!\!,\\
      &\quad~~~ j\leftarrow j+1,\nonumber
\end{align*}
where $\vect{\Pi}_{i,j}(k+1,0)=\vect{\Pi}_{i,j}(k)$,
$\vect{\Pi}_{i,j}(k+1,j)=\vect{\Pi}_{i,j}(k+1,j-1)+\vect{\Gamma}_i(k+1,j)\,
\vect{\Gamma}_j(k+1,j)^\top$. Here, $\vect{\Gamma}_i(k+1,j)$ is
calculated from~\eqref{eq::barD-i}-\eqref{eq::barD-b} wherein
$\vect{S}_{a,b}$ at each $j$ is calculated from~\eqref{eq::Sab_DCL}
using {$\Hvect{x}^{\prpg}(k+1)=\Hvect{x}^{\updt}(k+1,j-1)$} and
{$\vect{P}^{\prpg}(k+1)=\vect{P}^{\updt}(k+1,j-1)$}.  Consequently,
$\Bvect{r}^a(k+1,j)=\vect{S}_{a,b}(k+1,j)^{-\frac{1}{2}}\vect{r}^{a}(k\!+\!1,j)$. 
 The update at time {$k+1$} is
  $\Hvect{x}^{i\updt}(k+1)=\Hvect{x}^{i\updt}(k+1,n_s)$,
  $\vect{P}^{i\updt}(k+1)=\vect{P}^{i\updt}(k+1,n_s)$, and
  $\vect{P}_{i,l}^{\updt}(k+1)=\vect{P}_{i,l}^{\updt}(k+1,n_s)$,
  $i\in\VV$, $l\in\VV\backslash\{i\}$.

Notice
that we can represent the final updated variables as
\begin{subequations}\label{eq::sequ-updt-DCL}
\begin{align}
&\Hvect{x}^{i\updt}(k+1,n_s)=\Hvect{x}^{i\prpg}(k\!+\!1,0)+\\
&\quad\quad\vect{\Phi}^{i}(k\!+\!1)\sum\nolimits_{j=1}^{n_s}\vect{\Gamma}_i(k\!+\!1,j)\Bvect{r}^{a}(k\!+\!1,j),\nonumber\\
&\vect{P}^{i\updt}(k+1,n_s)=\vect{P}^{i\prpg}\!(k\!+\!1,0)-\\
&\quad\quad\!\vect{\Phi}^{i}(k\!+\!1) \Big(\sum\nolimits_{j=1}^{n_s}\vect{\Gamma}_{i}(k\!+\!1,j)\vect{\Gamma}_i^\top(k\!+\!1,j)\Big)\vect{\Phi}^i(k\!+\!1)\!^\top\!,\nonumber\\
&\vect{P}_{i,l}^{\updt}(k\!+\!1,n_s)=\\
&\quad\quad\vect{\Phi}^i(k+1)\Big(\sum\nolimits_{j=1}^{n_s}\,\vect{\Pi}_{i,l}(k+1,j-1)\, \Big)\vect{\Phi}^l(k+1)^\top.\nonumber
\end{align}
\end{subequations}
\setlength{\textfloatsep}{5pt}
\begin{algorithm}[!t]
{\scriptsize
\caption{{ Server's sequential updating procedure for multiple in-network measurement at time $k+1$}}
\label{alg::CCU_seqential}
\begin{algorithmic}[1]
\Require
 Initialization ($j=0$): server
obtains  the following information from each robot $a\in\VV_{\text{A}}(k+1)$ and all of its landmarks $b\in\VV_{\text{B}}^a(k+1)$, 
\begin{align*}
&\lmssg^a=\Big(\vect{z}_{a,b},\Hvect{x}^{a\prpg}(k+1), \vect{P}^{b\prpg}(k+1), \vect{\Phi}^{a}(k+1)\Big),\nonumber\\
&\lmssg_a^b=\Big(\Hvect{x}^{b\prpg}(k+1),\vect{P}^{b\prpg}(k+1), \vect{\Phi}^{b}(k+1)\Big).
\end{align*} 
The server initializes the following variables 
 \begin{align*}
    \Hvect{x}^{\updt i}(k\!+\!1,0)&=\Hvect{x}^{\prpg i}(k\!+\!1),~ \Hvect{P}^{\updt i}(k\!+\!1,0)=\vect{P}^{\prpg i}(k\!+\!1),~ \forall i\in\bar{\VV}(k\!+\!1),\\
    \vect{\Pi}_{i,l}(k\!+\!1,0)&=\vect{\Pi}_{i,l}(k), i\in\VV\backslash\{N\},~l\in\{i\!+\!1,\cdots,N\}.
    \end{align*}

\hspace{-0.38in}\noindent\textbf{Iteration $j$}: server proceeds with the following calculations.
\For{$a\in\VV_{\text{A}}(k+1)$} \For{$b\in\VV_{\text{B}}^a(k+1)$} \State {Server
  calculates $\Tvect{H}_a$, $\Tvect{H}_b$ and $\vect{r}^a$ using
  $\Hvect{x}^{\prpg a}(k+1)=\Hvect{x}^{\updt a}(k+1,j-1)$ and $\Hvect{x}^{\prpg b}(k+1)=\Hvect{x}^{\updt b}(k+1,j-1)$. Then,
  using these measurement matrices and $\Hvect{P}^{\prpg a}(k+1)=\Hvect{P}^{\updt a}(k+1,j-1)$,
  $\Hvect{P}^{\prpg b}(k+1)=\Hvect{P}^{\updt b}(k+1,j-1)$ and $\vect{\Pi}_{a,b}(k)=\vect{\Pi}_{a,b}(k+1,j-1)$, server
  calculates $\vect{S}_{a,b}$ from~\eqref{eq::Sab_DCL} and subsequently
  $\Bvect{r}^a(k+1,j)=(\vect{S}_{a,b}(k+1,j))^{-\frac{1}{2}}\vect{r}^a(k+1,j)$ and
  $\vect{\Gamma}_i(k+1,j)$ from~\eqref{eq::barD-i}-\eqref{eq::barD-b} for $i\in\VV$. Next, server
  updates the state and the covariance of all the robots in
  $i\in{\VV_{\text{A,B}}}(k+1)$ as follows
\begin{subequations}\label{eq::CCU_update_for_robots_in_V_A}
\begin{align}
          &   \Hvect{x}^{\updt i}(k\!+\!1,j)=\Hvect{x}^{\updt i}(k\!+\!1,j)\!+\!\vect{\Phi}^{i}(k\!+\!1) \,\vect{\Gamma}_i(k\!+\!1,j)\,\Bvect{r}^a\!(k\!+\!1,j),\\
          & \vect{P}^{\!\updt i}\!(k\!+\!1,j)\!=\!\vect{P}^{\!\updt
            i}(\!k\!+\!1,\!j)\!-\!\vect{\Phi}^{\!i}(k\!+\!1)\vect{\Gamma}_{i}(k\!+\!1,\!j)\!\vect{\Gamma}_i(k\!+\!1,\!j)\!^{\top}\!\!\vect{\Phi}^i\!(k\!+\!1)^{\!\top}\!\!.
 \end{align}
 \end{subequations}
 It also updates $\vect{\Pi}_{i,l}$ for
 $i\in\VV\backslash\{N\},~l\in\{i+1,\cdots,N\}$ as follows
 \begin{align*}
    \vect{\Pi}_{i,l}(k+1,j)=\vect{\Pi}_{i,l}(k+1&,j-1)-\vect{\Gamma}_i(k+1,j) \vect{\Gamma}_l(k+1,j)^\top,\quad\\
  &\text{if~}(i,l)\not\in\VV_{\text{missed}}(k+1)\times\VV_{\text{missed}}(k+1).
\end{align*}
 \State $j \leftarrow j+1$}
\EndFor
\EndFor
\State server sets  $ \vect{\Pi}_{i,l}(k+1)=\vect{\Pi}_{i,l}(k+1,n_s)$, where $n_s=\sum_{a\in\VV_{\text{A}}(k+1)}|\VV_\text{\text{B}}^a(k+1)|$.

\State server broadcasts the following update messages for robot $i\in\VV$
\begin{align}\label{eq:updt-mssg-squl}\textsl{update-message}^i=\,&\Big(\sum\nolimits_{j=1}^{n_s}\!(\vect{\Gamma}_i(k+1,j)\Bvect{r}^a(k+1,j)),\\
&~~\sum\nolimits_{j=1}^{n_s}\!(\vect{\Gamma}_i(k+1,j)\vect{\Gamma}_i(k+1,j)^\top)\Big).\nonumber\end{align}

\end{algorithmic}
}
\end{algorithm}
One can expect that the updating order must not dramatically change
the results (cf. ~\cite[page 104]{YB-PKW-XT:11} and references therein). Here, we assume that the server has a pre-specified \textsl{sequential-updating-order}
  guideline, which indicates the priority order for implementing the
  measurement update. To implement sequential updating procedure, the robots making
measurements inform the server and indicate to server what their landmark robots are., i.e.,  the server knows {$\VV_{\text{A}}(k+1)$} and
{$\VV_\text{\text{B}}^i(k+1)$}'s, and sorts both of these sets
according to it's sequential-updating-order guideline. The server collects all the landmark
messages~\eqref{eq::DCL-lmssg} of the robots in
{$\VV_{\text{A,B}}(k+1)$}. 
We use the compact representation~\eqref{eq::sequ-updt-DCL} of the
sequential updating procedure to develop a partially decentralized
implementation which requires only one update message broadcast from
the server, see Algorithm~\ref{alg::CCU_seqential}.
 Note that in this
implementation, the server should create a local copy of the state estimate and the error covariance equations of the robots in {${\VV_{\text{A,B}}}(k+1)$}
(see.~\eqref{eq::CCU_update_for_robots_in_V_A}), because these updates are needed to compute $\vect{S}_{a,b}$ and other intermediate variables. An alternative implementation is also possible where the
update message for every robot $i\in\VV_{\text{A},\text{B}}(k+1)$
is \begin{align*}
&  \textsl{update-message}^i=\Big((\vect{\Phi}^i)^{-1}(\Hvect{x}^{\updt
    i}(k+1,n_s)-\Hvect{x}^{\prpg
    i}(k+1)),\\
  &\quad\qquad-\!(\vect{\Phi}^i)^{-1}(\vect{P}^{\updt
    i}(k+1,n_s)\!-\!\vect{P}^{\prpg
    i}(k+1))(\vect{\Phi}^i)^{-T}\Big).  \end{align*} instead
of~\eqref{eq:updt-mssg-squl}. This is because the server already has
computed the update state estimates and the corresponding covariances
of robot $i\in{\VV_{\text{A,B}}}(k+1)$ as part of partial updating
procedure, i.e, $\Hvect{x}^{\updt i}(k+1)=\Hvect{x}^{\updt
  i}(k+1,n_s)$, and $\vect{P}^{\updt i}(k+1)=\vect{P}^{\updt
  i}(k+1,n_s)$.

\end{document}